\pdfoutput=1
\documentclass{article} 
\usepackage{times}
\usepackage{geometry}

\usepackage[utf8]{inputenc} 
\usepackage[T1]{fontenc}    
\usepackage[colorlinks,citecolor=blue,linkcolor=blue]{hyperref}       
\usepackage{url}            
\usepackage{booktabs}       
\usepackage{amsfonts}       
\usepackage{nicefrac}       
\usepackage{microtype}      

\usepackage{mkolar_definitions}
\usepackage{url}
\usepackage{authblk}
\usepackage{lipsum}

\usepackage{graphicx}
\usepackage{subfigure}
\usepackage{multirow}
\usepackage{floatrow}
\newfloatcommand{capbtabbox}{table}[][\FBwidth]

\PassOptionsToPackage{numbers}{natbib}
\usepackage{natbib}

\usepackage[ruled, vlined, linesnumbered]{algorithm2e}
\usepackage{tikz}
\usepackage{xspace}
\usepackage{graphicx}

\usepackage{enumerate}
\usepackage{amsthm,amsmath,amssymb}
\usepackage{amsfonts,dsfont}
\usepackage{nicefrac}
\usepackage{microtype}
\usepackage{mathtools}
\usepackage{footnote}

\usepackage{comment}
\newcount\Comments  
\definecolor{darkgreen}{rgb}{0,0.5,0}
\definecolor{darkred}{rgb}{0.7,0,0}
\definecolor{teal}{rgb}{0.3,0.8,0.8}
\newcommand{\kibitz}[2]{\ifnum\Comments=1{\textcolor{#1}{\textsf{\footnotesize #2}}}\fi}

\newcommand{\version}{arxiv}

\newenvironment{packed_enum}{
  \begin{enumerate}
    \setlength{\itemsep}{1pt}
    \setlength{\parskip}{-1pt}
    \setlength{\parsep}{0pt}
}{\end{enumerate}}

\newcounter{qcounter}
 {\end{list}}

\newcommand{\ellipsoid}{\gamma(T)}
\renewcommand{\order}{\ensuremath{O}}
\renewcommand{\otil}{\ensuremath{\tilde{\order}}}
\newcommand{\alg}{\textsc{bose}\xspace}
\newcommand{\reg}{\textrm{Reg}}
\newcommand{\linucb}{\textsc{OFUL}\xspace}
\newcommand{\monster}{\textsc{ILTCB}\xspace}
\newcommand{\greedy}{\textsc{EpsGreedy}\xspace}
\newcommand{\thompson}{\textsc{Thompson}\xspace}
\def\tran{^\top}

\usepackage{prettyref}
\newcommand{\pref}[1]{\prettyref{#1}}

\newcommand{\savehyperref}[2]{\texorpdfstring{\hyperref[#1]{#2}}{#2}}
\newrefformat{eq}{\savehyperref{#1}{\textup{(\ref*{#1})}}}
\newrefformat{eqn}{\savehyperref{#1}{Equation~\ref*{#1}}}
\newrefformat{lem}{\savehyperref{#1}{Lemma~\ref*{#1}}}
\newrefformat{def}{\savehyperref{#1}{Definition~\ref*{#1}}}
\newrefformat{line}{\savehyperref{#1}{line~\ref*{#1}}}
\newrefformat{thm}{\savehyperref{#1}{Theorem~\ref*{#1}}}
\newrefformat{corr}{\savehyperref{#1}{Corollary~\ref*{#1}}}
\newrefformat{sec}{\savehyperref{#1}{Section~\ref*{#1}}}
\newrefformat{ssec}{\savehyperref{#1}{Subsection~\ref*{#1}}}
\newrefformat{app}{\savehyperref{#1}{Appendix~\ref*{#1}}}
\newrefformat{ass}{\savehyperref{#1}{Assumption~\ref*{#1}}}
\newrefformat{as}{\savehyperref{#1}{Assumption~\ref*{#1}}}
\newrefformat{ex}{\savehyperref{#1}{Example~\ref*{#1}}}
\newrefformat{fig}{\savehyperref{#1}{Figure~\ref*{#1}}}
\newrefformat{alg}{\savehyperref{#1}{Algorithm~\ref*{#1}}}
\newrefformat{rem}{\savehyperref{#1}{Remark~\ref*{#1}}}
\newrefformat{conj}{\savehyperref{#1}{Conjecture~\ref*{#1}}}
\newrefformat{prop}{\savehyperref{#1}{Proposition~\ref*{#1}}}
\newrefformat{proto}{\savehyperref{#1}{Protocol~\ref*{#1}}}
\newrefformat{prob}{\savehyperref{#1}{Problem~\ref*{#1}}}
\newrefformat{claim}{\savehyperref{#1}{Claim~\ref*{#1}}}
\newrefformat{fact}{\savehyperref{#1}{Fact~\ref*{#1}}}

\begin{document} 

\title{Semiparametric Contextual Bandits}

\author[2]{
  Akshay Krishnamurthy
  \thanks{akshay@cs.umass.edu}}
\author[2]{
  Zhiwei Steven Wu
  \thanks{zsw@umn.edu}}
\author[3]{
  Vasilis Syrgkanis
  \thanks{vasy@microsoft.com}}

\affil[2]{Microsoft Research NYC, New York, NY}
\affil[3]{Microsoft Research New England, Cambridge, MA}
\maketitle

\begin{abstract}
  This paper studies \emph{semiparametric contextual bandits}, a
  generalization of the linear stochastic bandit problem where the
  reward for an action is modeled as a linear function of known action
  features confounded by a non-linear action-independent term. We
  design new algorithms that achieve $\otil(d\sqrt{T})$ regret over
  $T$ rounds, when the linear function is $d$-dimensional, which
  matches the best known bounds for the simpler unconfounded case and
  improves on a recent result of~\citet{greenewald2017action}.
  Via an empirical evaluation, we show that our algorithms outperform
  prior approaches when there are non-linear confounding effects on
  the rewards.
  Technically, our algorithms use a new reward estimator inspired by
  doubly-robust approaches and our proofs require new concentration
  inequalities for self-normalized martingales.






\end{abstract}

\section{Introduction}
\label{sec:introduction}

A number of applications including online personalization, mobile
health, and adaptive clinical trials require that an agent repeatedly
makes decisions based on user or patient information with the goal of
optimizing some metric, typically referred to as a reward. For
example, in online personalization problems, we might serve content
based on user history and demographic information with the goal of
maximizing user engagement with our service. Since counterfactual
information is typically not available, these problems require
algorithms to carefully balance \emph{exploration}---making
  potentially suboptimal decisions to acquire new information---with
\emph{exploitation}---using collected information to make
  better decisions. Such problems are often best modeled with the
framework of \emph{contextual bandits}, which captures the
exploration-exploitation tradeoff and enables rich decision making
policies but ignores the long-term temporal effects that make general
reinforcement learning challenging. Contextual bandit algorithms have
seen recent success in applications, including news
recommendation~\cite{li2010contextual} and mobile
health~\cite{tewari2017ads}.

Contextual bandit algorithms can be categorized as either
\emph{parametric} or \emph{agnostic}, depending on whether they model
the relationship between the reward and the decision or
not. Parametric approaches typically assume that the reward is a
(generalized) linear function of a known decision-specific feature
vector~\citep{filippi2010parametric,chu2011contextual,abbasi2011improved,agrawal2013thompson}. Once
this function is known to high accuracy, it can be used to make
near-optimal decisions. Exploiting this fact, algorithms for this
setting focus on learning the parametric model. Unfortunately, fully
parametric assumptions are often unrealistic and challenging to verify
in practice, and these algorithms may perform poorly when the
assumptions do not hold.


In contrast, agnostic approaches make no modeling assumptions about
the reward and instead compete with a large class of decision-making
policies~\citep{langford2008epoch,agarwal2014taming}. While these
policies are typically parametrized in some way, these algorithms
provably succeed under weaker conditions and are generally more robust
than parametric ones. On the other hand, they typically have
worse statistical guarantees, are conceptually much more complex, and
have high computational overhead, technically requiring solving
optimization problems that are NP-hard in the worst case. This leads
us to a natural question:
\begin{quote}
\emph{Is there an algorithm that inherits the simplicity and
  statistical guarantees of the parametric methods \emph{and} the
  robustness of the agnostic ones?}
\end{quote}
Working towards an affirmative answer to this question, we consider a
semiparametric contextual bandit setup where the reward is modeled as
a linear function of the decision confounded by an additive non-linear
perturbation that is independent of the decision. This setup
significantly generalizes the standard parametric one, allowing for
complex, non-stationary, and non-linear rewards (See~\pref{sec:preliminaries} for a precise formulation). On the
other hand, since this perturbation is just a baseline reward for all
decisions, it has no influence on the optimal one, which depends
only on the unknown linear function. In the language of econometrics
and causal modeling, the \emph{treatment effect} is linear.

In this paper, we design new algorithms for the semiparametric
contextual bandits problem. When the linear part of the reward is
$d$-dimensional, our algorithms achieve $\otil(d\sqrt{T})$ regret over
$T$ rounds, even when the features and the confounder are chosen by an
adaptive adversary. This guarantee matches the best results for the
simpler linear stochastic bandit problem up to logarithmic terms,
showing that there is essentially no statistical price to pay for
robustness to confounding effects. On the other hand, our algorithm
and analysis is quite different, and it is not hard to see that
existing algorithms for stochastic bandits fail in our more general
setting. Our regret bound also improves on a recent result
of~\citet{greenewald2017action}, who consider the same setup but study
a weaker notion of regret. Our algorithm, main theorem, and
comparisons are presented in~\pref{sec:algorithm}.

We also compare our algorithm to approaches from both parametric and
agnostic families in an empirical study (we use a linear policy class
for agnostic approaches). In~\pref{sec:experiments}, we
evaluate several algorithms on synthetic problems where the reward is
(a) linear, and (b) linear with confounding. In the linear case, our
approach learns, but is slightly worse than the baselines. On the
other hand, when there is confounding, our algorithm significantly
outperforms both parametric and agnostic approaches. As such, these
experiments demonstrate that our algorithm represents a favorable
trade off between statistical efficiency and robustness.

On a technical level, our algorithm and analysis require several new
ideas. First, we derive a new estimator for linear models in the
presence of confounders, based on recent and classical work in
semiparametric statistics and
econometrics~\cite{robinson1988root,chernozhukov2016double}. Second,
since standard algorithms using optimism principles fail to guarantee
consistency of this new estimator, we design a new randomized
algorithm, which can be viewed as an adaptation of the
action-elimination method of~\citet{EMM06} to the contextual bandits
setting. Finally, analyzing the semiparametric estimator requires an
intricate deviation argument, for which we derive a new
self-normalized inequality for vector-valued martingales using tools
from~\citet{delapena2008self,delapena2009theory}.

\section{Preliminaries}
\label{sec:preliminaries}
We study a generalization of the linear stochastic bandit problem with
\emph{action-dependent} features and \emph{action-independent}
confounder. The learning process proceeds for $T$ rounds, and in round
$t$, the learner receives a context
$x_t \triangleq \{z_{t,a}\}_{a \in \Acal}$ where $z_{t,a} \in \RR^d$
and $\Acal$ is the action set, which we assume to be large but
finite. The learner then chooses an action $a_t \in \Acal$ and
receives reward
\begin{align}
  r_t(a_t) \triangleq \langle \theta, z_{t,a_t}\rangle + f_t(x_t) + \xi_t,
\end{align}
where $\theta \in \RR^d$ is an unknown parameter vector, $f_t(x_t)$ is
a confounding term that depends on the context $x_t$ but, crucially,
does not depend on the chosen action $a_t$, and $\xi_t$ is a noise
term that is centered and independent of $a_t$.

For each round $t$, let
$a_t^\star\triangleq\argmax_{a \in \Acal} \langle \theta,
z_{t,a}\rangle$ denote the optimal action for that round. The goal of
our algorithm is to minimize the regret, defined as
\begin{align*}
\textrm{Reg}(T) \triangleq \sum_{t=1}^T r_t(a_t^\star) - r_t(a_t)
= \sum_{t=1}^T\langle\theta, z_{t,a_t^\star} - z_{t,a_t}\rangle.
\end{align*}
Observe that the noise term $\xi_t$, and, more importantly, the
confounding term $f_t(x_t)$ are absent in the final expression, since
they are independent of the action choice.

We consider the challenging setting where the context $x_t$ and the
confounding term $f_t(\cdot)$ are chosen by an adaptive adversary, so
they may depend on all information from previous rounds. This is
formalized in the following assumption.
\begin{assum}[Environment]
\label{as:env}
We assume that $x_t = \{z_{t,a}\}_{a \in \Acal}, f_t, \xi_t$ are
generated at the beginning of round $t$, before $a_t$ is chosen. We
assume that $x_t$ and $f_t$ are chosen by an adaptive adversary, and
that $\xi_t$ satisfies $\EE[\xi_t | x_t,f_t] = 0$ and $|\xi_t| \le 1$.
\end{assum}

We also impose mild regularity assumptions on the parameter, the
feature vectors, and the confounding functions.
\begin{assum}[Boundedness]
\label{as:bd}
Assume that $\|\theta\|_2 \le 1$ and that $\|z_{t,a}\|_2 \le 1$ for
all $a \in \Acal, t \in [T]$. Further assume that
$f_t(\cdot) \in [-1,1]$ for all $t\in [T]$.
\end{assum}
For simplicity, we assume an upper bound of $1$ in these conditions,
but our algorithm and analysis can be adapted to more generic
regularity conditions.

\ifthenelse{\equal{\version}{workshop}}{}{
\paragraph{Related work.}
Our setting is related to linear stochastic bandits and several
variations that have been studied in recent years. Among these, the
closest is the work of~\citet{greenewald2017action} who consider the
same setup and provide a Thompson Sampling algorithm using a new
reward estimator that eliminates the confounding term. Motivated by
applications in medical intervention, they consider a different notion
of regret from our more-standard notion and, as such, the results are
somewhat incomparable. For our notion of regret, their analysis can
produce a $T^{2/3}$-style regret bound, which is worse than our
optimal $\sqrt{T}$ bound. See~\pref{sec:theorem} for a more detailed
comparison.



Other results for linear stochastic bandits include upper-confidence
bound
algorithms~\cite{rusmevichientong2010linearly,chu2011contextual,abbasi2011improved},
Thompson sampling
algorithms~\citep{agrawal2013thompson,russo2014learning}, and
extensions to generalized linear
models~\citep{filippi2010parametric,li2017provable}. However, none of
these models accommodate arbitrary and non-linear confounding
effects. Moreover, apart from Thompson sampling, all of these
algorithms use deterministic action-selection policies (conditioning
on the history), which provably incurs $\Omega(T)$ regret in our
setting, as we will see.

One can accommodate confounded rewards via an agnostic-learning
approach to contextual
bandits~\cite{auer2002nonstochastic,langford2008epoch,agarwal2014taming}. In
this framework, we make no assumptions about the reward, but rather
compete with a class of parameterized policies (or experts). Since a
$d$-dimensional linear policy is optimal in our setting,
an agnostic algorithm with a linear policy class addresses precisely
our notion of regret. However there are two disadvantages. First,
agnostic algorithms are all computationally intractable, either
because they enumerate the (infinitely large) policy class, or because
they assume access to optimization oracles that can solve NP-hard
problems in the worst case. Second, most agnostic approaches have
regret bounds that grow with $\sqrt{K}$, the number of actions, while
our bound is completely independent of $K$.

We are aware of one approach that is independent of $K$, but it
requires enumeration of an infinitely large policy class. This method
is based on ideas from the adversarial linear and combinatorial
bandits
literature~\citep{dani2008stochastic,abernethy2009competing,bubeck2012towards,cesa2012combinatorial}. Writing
$\theta_t \triangleq (\theta, f_t(x_t)) \in \RR^{d+1}$ and $z'_{t,a} \triangleq (z_{t,a},1) \in \RR^{d+1}$, our
setting can be re-formulated in the adversarial linear bandits
framework. However, standard linear bandit algorithms compete with the
best fixed action vector in hindsight, rather than the best policy
with time-varying action sets.
To resolve this, one can use the linear bandits reward
estimator~\citep{swaminathan2017off} in a contextual bandit algorithm
like EXP4~\citep{auer2002nonstochastic}, but this approach is not
computationally tractable with the linear policy class. For our
setting, we are not aware of any computationally efficient approaches,
even oracle-based approaches, that achieve $\textrm{poly}(d)\sqrt{T}$
regret with no dependence on the number of actions.

We resolve the challenge of confounded rewards with an estimator from
the semiparametric statistics
literature~\citep{tsiatis2007semiparametric}, which focuses on
estimating functionals of a nonparametric model. Most estimators are
based on \emph{Neyman Orthogonalization}~\citep{neyman1979c}, which
uses moment equations that are insensitive to nuisance parameters in a
method-of-moments approach~\citep{chernozhukov2016double}. These
orthogonal moments typically involve a linear correction to an initial
nonparametric estimate using so-called \emph{influence
  functions}~\citep{bickel1998efficient,robins2008higher}. \citet{robinson1988root}
used this approach for the offline version of our setting (known as
the partially linear regression (PLR) model) where he demonstrated a
form of \emph{double-robustness}~\citep{robins1992recovery} to poor
estimation of the nuisance term (in our case $f_t(x_t)$). We
generalize Robinson's work to the online setting, showing how
orthogonalized estimators can be used for adaptive exploration. This
requires several new techniques, including a novel action selection
policy and a self-normalized inequality for vector-valued martingales.
}

\section{Algorithm and Results}
\label{sec:algorithm}
\ifthenelse{\equal{\version}{workshop}}{}
{In this section, we describe our algorithm and present our main
theoretical result, an $\otil(d\sqrt{T})$ regret bound for the
semiparametric contextual bandits problem.
}
\ifthenelse{\equal{\version}{workshop}}{
}{
\subsection{A Lower Bound}
Before turning to the algorithm, we first present a lower bound
against deterministic algorithms. Since the functions $f_t$ may be
chosen by an adaptive adversary, it is not hard to show that this
setup immediately precludes the use of deterministic algorithms.
\begin{proposition}
  \label{prop:det_lb}
  Consider an algorithm that, at round $t$, chooses an action $a_t$ as
  a deterministic function of the observable history
  $H_t \triangleq \{x_{1:t}, a_{1:t-1}, r_{1:t-1}\}$. There exists a
  semiparametric contextual bandit instance with $d=2$ and $K=2$ where
  the regret of the algorithm is at least $T/2$.
\end{proposition}

See~\pref{app:lower} for the proof, which resembles the standard
argument against deterministic online learning
algorithms~\citep{cover1965behavior}.  The main difference is that the
adversary uses the confounding term to corrupt the information that
the learner receives, whereas, in the standard proof, the adversary
chooses the optimal action in response to the learner.
In fact, deterministic algorithms can
succeed in the full information version of our setting, since taking
differences between rewards eliminates the confounder. Thus, bandit
feedback plays a crucial role in our construction and the bandit
setting is considerably more challenging than the full information
analogue.


We emphasize that, except for the Thompson Sampling
approach~\cite{agrawal2013thompson}, essentially all algorithms for
the linear stochastic bandit problem use deterministic strategies, so
they provably fail in the semiparametric setting. As we mentioned,
Thompson Sampling was analyzed in our setting
by~\citet{greenewald2017action}, but they do not obtain the optimal
$\sqrt{T}$-type regret bound (See~\pref{sec:theorem} for a more
quantitative and detailed comparison). In contrast, our algorithm is
quite different from all of these approaches; it ensures enough
randomization to circumvent the lower bound and also achieves the
optimal $\sqrt{T}$ regret.

To conclude this discussion, we remark that the $\Omega(d\sqrt{T})$
lower bound for linear stochastic bandits~\citep{dani2008stochastic},
which also applies to randomized algorithms, holds in our more general
setting as well.

\subsection{The Algorithm}
}
\begin{algorithm*}[t]
\caption{\alg (Bandit orthogonalized semiparametric estimation)}
\label{alg:pseudocode}
\SetKwInOut{Inputa}{Input}
  \Inputa{$T, \delta \in (0,1)$.}
  Set $\lambda \gets 4d\log(9T) + 8\log(4T/\delta)$ and $\ellipsoid \gets \sqrt{\lambda} + \sqrt{27d\log(1+2T/d) + 54\log(4T/\delta)}$. \\
  Initialize $\hat{\theta} \gets 0 \in \RR^d, \Gamma \gets \lambda I_{d \times d}$.\\
  \For{$t = 1,\ldots,T$}{
    Observe $x_t = \{z_{t,a}\}_{a \in \Acal}$\\
    Filter
    \vspace{-0.3cm}
    \begin{align}
      \Acal_t \gets \left\{ a \in \Acal \mid \forall b \in \Acal, \langle \hat{\theta}, z_{t,b} - z_{t,a} \rangle \le \ellipsoid \|z_{t,a} - z_{t,b}\|_{\Gamma^{-1}}\right\}. \label{eq:filter}
    \end{align}\\
    \vspace{-0.3cm}
    Find distribution $\pi_t \in \Delta(\Acal_t)$ such that $\forall a \in \Acal_t$ (We use $\Cov_{b \sim \pi_t}(z_{t,b}) \triangleq \EE[ z_{t,b}z_{t,b}\tran] - (\EE z_{t,b})(\EE z_{t,b})\tran$.)
    \vspace{-0.3cm}
    \begin{align}
      \|z_{t,a} - \EE_{b\sim \pi_t} z_{t,b}\|_{\Gamma^{-1}}^2 \leq \tr(\Gamma^{-1}\Cov_{b \sim \pi_t}(z_{t,b})). \label{eq:optimize}
    \end{align}\\
    \vspace{-0.3cm}
    Sample $a_t \sim \pi_t$ and play $a_t$. Observe $r_t(a_t)$. ($r_t(a_t) = \langle \theta, z_{t,a_t}\rangle + f_t(x_t) + \xi_t$.)\\
    Let $\mu_t = \EE_{a \sim \pi_t}[ z_{t,a} \mid x_t]$ and update parameters
    \vspace{-0.3cm}
    \begin{align}
      \Gamma \gets \Gamma + (z_{t,a_t} - \mu_t)(z_{t,a_t} - \mu_t)\tran, \qquad
      \hat{\theta} \gets \Gamma^{-1}\sum_{\tau=1}^t (z_{\tau,a_\tau} - \mu_\tau)r_\tau(a_\tau). \label{eq:estimator}
    \end{align}
    \vspace{-0.5cm}
  }
\end{algorithm*}

Pseudocode for the algorithm, which we call \alg, for ``Bandit
Orthogonalized Semiparametric Estimation," is displayed in~\pref{alg:pseudocode}.  The algorithm maintains an estimate
$\hat{\theta}$ for the true parameter $\theta$, which it uses in each
round to select an action via two steps: (1) an action elimination
step that removes suboptimal actions, and (2) an optimization step
that finds a good distribution over the surviving actions. The
algorithm then samples and plays an action from this distribution and
uses the observed reward to update the parameter estimate
$\hat{\theta}$. This parameter estimation step is the third main
element of the algorithm. 
\ifthenelse{\equal{\version}{workshop}}{}{
We now describe each of these three
components in detail.

\paragraph{Parameter estimation.} For simplicity, we use
$z_t \triangleq z_{t,a_t}$ to denote the feature vector for the action that was
chosen at round $t$, and similarly we use $r_t \triangleq r_t(a_t)$. Using all
previously collected data, specifically
$\{z_\tau, r_\tau\}_{\tau=1}^t$ at the end of round $t$, we would like
to estimate the parameter $\theta$. First, if $f_\tau(x_\tau)$ were
identically zero, by exploiting the linear parametrization we could
use ridge regression, which with some $\lambda > 0$ gives
\begin{align*}
  \hat{\theta}_{\textrm{Ridge}} \triangleq \left(\lambda I + \sum_{\tau=1}^t z_\tau z_\tau\tran\right)^{-1}\sum_{\tau=1}^t z_\tau r_\tau.
\end{align*}
This estimator appears in most prior approaches for linear stochastic
bandits~\cite{rusmevichientong2010linearly,chu2011contextual,abbasi2011improved}. Unfortunately,
since $f_\tau(x_\tau)$ is non-zero, $\hat{\theta}_{\textrm{Ridge}}$
has non-trivial and non-vanishing bias, so even in benign settings it
is not a consistent estimator for $\theta$.\footnote{A related
  estimator \emph{can} be used to evaluate the reward of a policy, as
  in linear and combinatorial bandits~\citep{cesa2012combinatorial},
  but to achieve adequate exploration, one must operate over the
  policy class, which leads to computational intractability. We would
  like to use $\hat{\theta}$ to drive exploration, and this seems to
  require a consistent estimator. See~\pref{app:bias} for a
  simple example demonstrating how using a biased estimator in a
  confidence-based approach results in linear regret.}

Our approach to eliminating the bias from the confounding term
$f_\tau(x_\tau)$ is to center the feature vectors
$z_\tau$. Intuitively, in the ridge estimator, if $z_\tau$ is
centered, then $z_\tau(r_\tau - \langle\theta^\star,z_\tau\rangle)$ is mean zero, even when
there is non-negligible bias in the second term. As such, the error
of the corresponding estimator can be expected to concentrate around zero. In the semiparametric
statistics literature, this is known as \emph{Neyman
  Orthogonalization}~\citep{neyman1979c}, which was analyzed in the
context of linear regression by~\citet{robinson1988root} and in a more
general setting by~\citet{chernozhukov2016double}.

To center the feature vector, we will, at round $t$, choose action
$a_t$ by sampling from some distribution $\pi_t \in
\Delta(\Acal)$. Let
$\mu_t \triangleq \EE_{a_t \sim \pi_t}[ z_{t,a_t}|x_t]$ denote the
mean feature vector, taking expectation only over our random action
choice. With this notation, the orthogonalized estimator is
\ifthenelse{\equal{\version}{arxiv}}{
\begin{align*}
  \Gamma &= \lambda I + \sum_{\tau=1}^t(z_\tau- \mu_\tau)(z_\tau-\mu_\tau)\tran, \qquad
  \hat{\theta}  = \Gamma^{-1}\sum_{\tau=1}^t(z_\tau-\mu_\tau)r_\tau.
\end{align*}
}{
\begin{align*}
  \Gamma &= \lambda I + \sum_{\tau=1}^t(z_\tau- \mu_\tau)(z_\tau-\mu_\tau)\tran,\\
  \hat{\theta} & = \Gamma^{-1}\sum_{\tau=1}^t(z_\tau-\mu_\tau)r_\tau.
\end{align*}
}
$\hat{\theta}$ is a Ridge regression version of Robinson's classical
semiparametric regression estimator~\citep{robinson1988root}.
The estimator was originally derived for observational studies where
one might not know the \emph{propensities} $\mu_\tau$ exactly, and the
standard description involves estimates $\hat{f}_\tau$ and
$\hat{\mu}_\tau$ for the confounding term $f_\tau$ and the
propensities $\mu_\tau$ respectively. Informally, the estimator
achieves a form of double-robustness, in the sense that it is accurate
if either of these auxilliary estimators are.
In our case, since we know the propensities $\mu_\tau$ exactly, we can
use an inconsistent estimator for the confounding term, so we simply set $\hat{f}_\tau(x_\tau) \equiv 0$. 
\ifthenelse{\equal{\version}{workshop}}{We}{In~\pref{lem:estimator}, we} prove a precise finite sample
concentration inequality for this orthogonalized estimator, showing
that the confounding term $f_t(x_t)$ does not introduce any
bias. While the estimator has been studied in prior
works~\citep{robinson1988root}, to our knowledge, our error guarantee
is novel.




The convergence rate of the orthogonalized estimator depends on the
eigenvalues of the matrix $\Gamma$, and we must carefully select
actions to ensure these eigenvalues are sufficiently large. To see
why, notice that any deterministic action-selection approach with the
orthogonalized estimator (including confidence based approaches), will
fail, since $z_t = \mu_t$, so the eigenvalues of $\Gamma$ do not grow
rapidly and in fact the estimator is identically $0$. This argument
motivates our new action selection scheme which ensure substantial
conditional covariance.

\paragraph{Action selection.}
Our action selection procedure has two main elements. First using our
estimate $\hat{\theta}$, we eliminate any action that is provably
suboptimal. Based on our analysis for the estimator $\hat{\theta}$, at
round $t$, we can certify action $a$ is suboptimal, if we can find
another action $b$ such that
\begin{align*}
  \langle \hat{\theta}, z_{t,b} - z_{t,a} \rangle > \ellipsoid \|z_{t,b} - z_{t,a}\|_{\Gamma^{-1}}.
\end{align*}
Here $\ellipsoid$ is the constant specified in the algorithm, and
$\|x\|_{M} \triangleq \sqrt{x\tran Mx}$ denotes the Mahalanobis norm. Using our
confidence bound for $\hat{\theta}$\ifthenelse{\equal{\version}{workshop}}{}{ in~\pref{lem:estimator}
below}, this inequality certifies that action $b$ has higher expected
reward than action $a$, so we can safely eliminate $a$ from
consideration.

The next component is to find a distribution over the surviving
actions, denoted $\Acal'_t$ at round $t$, with sufficient
covariance. The distribution $\pi_t \in \Delta(\Acal'_t)$ that we use
is the solution to the following feasibility problem
\begin{align*}
\forall a \in \Acal_{t}', \quad  \|z_{t,a} - \EE_{b \sim \pi_t}z_{t,b}\|_{\Gamma^{-1}}^2 \leq \tr(\Gamma^{-1}\Cov_{b\sim\pi_t}(z_{t,b})).
\end{align*}
For intuition, the left hand side of the constraint for action $a$ is
an upper bound on the expected regret if $a$ is the optimal action on
this round. Thus, the constraints ensure that the regret is related to
the covariance of the distribution, which means that if we incur high
regret, the covariance term $\Cov_{b\sim\pi_t}(z_{t,b})$ will be
large. Since we use a sample from $\pi_t$ to update our parameter
estimate, this means that whenever the instantaneous regret is large,
we must learn substantially about the parameter.  In this way, the
distribution $\pi_t$ balances exploration and exploitation. We will
see \ifthenelse{\equal{\version}{workshop}}{}{in~\pref{lem:duality}} that this program is convex and always
has a feasible solution.

\ifthenelse{\equal{\version}{workshop}}{}{
Our action selection scheme bears some resemblance to
action-elimination approaches that have been studied in various bandit
settings~\cite{EMM06}. The main differences are that we adapt these
ideas to the contextual setting and carefully choose a distribution
over the surviving actions to balance exploration and exploitation.
}

\subsection{The Main Result}
\label{sec:theorem}
}We now turn to the main result, a regret guarantee for \alg.

\begin{theorem}
  \label{thm:regret}
  Consider the semiparametric contextual bandit problem under
  \pref{as:env} and~\pref{as:bd}. For any parameter
  $\delta \in (0,1)$, with probability at least $1-\delta$,
  \pref{alg:pseudocode} has regret at most
  $O(d\sqrt{T} \log(T/\delta))$.
\end{theorem}

\ifthenelse{\equal{\version}{workshop}}{}{
The constants, and indeed a bound depending on $\lambda$ and
$\ellipsoid$ can be extracted from the proof, provided in the
appendix.  To interpret the regret bound, it is worth comparing with
several related results:

\paragraph{Comparison with linear stochastic bandits.}
While most algorithms for linear stochastic bandits provably fail in
our setting (via~\pref{prop:det_lb}), the best regret
bounds here are $O(\sqrt{dT\log
  (TK/\delta)})$~\citep{chu2011contextual} and
$O(d\sqrt{T}\log(T)+\sqrt{dT\log(T)\log(1/\delta)})$~\citep{abbasi2011improved}
depending on whether we assume that the number of actions $K$ is small
or not. This latter result is optimal when the number of actions is
large~\citep{dani2008stochastic}, which is the setting we are
considering here.
Since our bound matches this optimal regret up to logarithmic factors,
and since linear stochastic bandits are a special case of our
semiparametric setting, our result is therefore also optimal up to
logarithmic factors. An interesting open question is whether an
$\otil(\sqrt{dT\log(K/\delta)})$ regret bound is achievable in the
semiparametric setting.

\paragraph{Comparison with agnostic contextual bandits.}
The best
oracle-based agnostic approaches achieve $\otil(\sqrt{dKT})$ regret~\cite{agarwal2014taming}, incurring
a polynomial dependence on the number of actions $K$, although there
is one inefficient method that can achieve
$\otil(d\sqrt{T})$,\footnote{This follows
  easily by combining ideas from~\citet{auer2002nonstochastic}
  and~\citet{cesa2012combinatorial}.} as we discussed previously. To
date, all efficient methods in the agnostic setting require some form
of i.i.d.~\citep{agarwal2014taming} or transductive
assumption~\citep{syrgkanis2016efficient,rakhlin2016bistro} on the
contexts, which we do not assume here.

\paragraph{Comparison with~\citet{greenewald2017action}.}
\citet{greenewald2017action} consider a very similar setting to ours,
where rewards are linear with confounding, but where one default
action $a_0$ always has $z_{t,a_0} \equiv 0 \in \RR^d$. Applications in
mobile health motivate a restriction that the algorithm choose the
$a_0$ action with probability $\in [p,1-p]$ for some small
$p \in (0,1)$. Their work also introduces a new notion of regret where
they compete with the policy that also satisfies this constraint but
otherwise chooses the optimal action $a_t^\star$. In this setup, they
obtain an $\otil(d^2\sqrt{T})$ regret bound, which has a worse
dimension dependence than~\pref{thm:regret}.

While the setup is somewhat different, we can still translate our
result into a regret bound in their setting, since \alg can support
the probability constraint, and by coupling the randomness between
\alg and the optimal policy, the regret is
unaffected.\footnote{Technically it is actually smaller by a factor of
  $(1-p)$.} On the other hand, since the constant in their regret
bound scales with $1/p$, their results as stated are vacuous when
$p=0$ which is precisely our setting. For our more challenging regret
definition, their analysis can produce a suboptimal $T^{2/3}$-style
regret bound, and in this sense,~\pref{thm:regret} provides a
quantitative improvement.


\paragraph{Summary.} \alg achieves essentially the same regret bound
as the best linear stochastic bandit methods, but in a much more
general setting. On the other hand, the agnostic methods succeed under
even weaker assumptions, but have worse regret guarantees and/or are
computationally intractable. Thus, \alg broadens the scope for
computationally efficient contextual bandit learning.
}

\ifthenelse{\equal{\version}{workshop}}{}
{
\section{Proof Sketch}
We sketch the proof of~\pref{thm:regret} in the
two-action case ($|\Acal|=2$), which has a much simpler proof that
preserves the main ideas. The technical machinery needed for the
general case is much more sophisticated, and we briefly describe some
of these steps at the end of this section, with a complete proof in
the Appendix.

In the two arm case, one should set
$\ellipsoid \triangleq \sqrt{\lambda} + \sqrt{9d\log(1+T/(d\lambda)) +
  18\log(T/\delta)}$ and $\lambda = O(1)$, which differs slightly from
the algorithm pseudocode for the more general case. Additionally, note
that with two actions, the uniform distribution over $\Acal_t$ is
always feasible for Problem~\pref{eq:optimize}. Specifically, if the
filtered set has cardinality $1$, we simply play that action
deterministically, otherwise we play one of the two actions uniformly
at random.

The proof has three main steps. First we analyze the orthogonalized
regression estimator defined in~\pref{eq:estimator}. Second, we study
the action selection mechanism and relate the regret incurred to the
error bound for the orthogonalized estimator. Finally, using a
somewhat standard potential argument, we show how this leads to a
$\sqrt{T}$-type regret bound. For the proof, let
$\hat{\theta}_t, \Gamma_t$ be the estimator and covariance matrix used
on round $t$, both based on $t-1$ samples.

For the estimator, we prove the following lemma for the
two action case. The main technical ingredient is a self-normalized
inequality for vector-valued martingales, which can be obtained using
ideas from~\citet{delapena2009theory}.
\begin{lemma}
  \label{lem:estimator}
  Under~\pref{as:env} and~\pref{as:bd}, let $K=2$ and
  $\ellipsoid \triangleq \sqrt{\lambda} + \sqrt{9 d \log(1+
    T/(d\lambda)) + 18\log(T/\delta)}$. Then, with probability at
  least $1-\delta$, the following holds simultaneously for all
  $t \in [T]$:
  \begin{align*}
    \|\hat{\theta}_t - \theta\|_{\Gamma_t} \leq \ellipsoid.
  \end{align*}
\end{lemma}
\begin{proof}
  Using the definitions and~\pref{as:env}, it is not hard to re-write
  \begin{align*}
    \hat{\theta}_t = \Gamma_t^{-1}(\Gamma_t - \lambda I)\theta + \Gamma_t^{-1}\sum_{\tau=1}^{t-1}Z_\tau \zeta_\tau,
  \end{align*}
  where $Z_\tau\triangleq z_{\tau,a_\tau} - \mu_\tau$ and
  $\zeta_\tau \triangleq \langle\theta,\mu_\tau\rangle + f_\tau(x_\tau) +
  \xi_\tau$. Further define
  $S_t \triangleq \sum_{\tau=1}^{t-1}Z_\tau \zeta_\tau$. Then, applying the
  triangle inequality the error is at most
  \begin{align*}
    \|\hat{\theta}_t - \theta\|_{\Gamma_t} \leq \|\lambda \theta\|_{\Gamma_t^{-1}} + \|S_t\|_{\Gamma_t^{-1}}.
  \end{align*}
  The first term here is at most $\sqrt{\lambda}$ since
  $\Gamma_t \succeq \lambda I$. To control the second term, we need to
  use a self-normalized concentration inequality, since $Z_\tau$ is a
  random variable, and the normalizing term
  $\Gamma_t = \lambda I + \sum_{\tau=1}^{t-1}Z_\tau Z_\tau\tran$ depends
  on the random realizations. In
  \pref{lem:self_normalized_symmetric} in the appendix, we prove
  that with probability at least $1-\delta$, for all $t \in [T]$
  \begin{align}
    \|S_t\|_{\Gamma_t^{-1}}^2 \leq 9d \log(1+T/(d\lambda)) + 18\log(T/\delta). \label{eq:two_arm_normalized}
  \end{align}
  The lemma follows from straightforward calculations.
\end{proof}

Before proceeding, it is worth commenting on the difference between
our self-normalized inequality~\pref{eq:two_arm_normalized} and a
slightly different one used by~\citet{abbasi2011improved} for the
linear case. In their setup, they have that $\zeta_\tau$ is
conditionally centered and sub-Gaussian, which simplifies the argument
since after fixing the $Z_\tau$s (and hence $\Gamma_t$), the
randomness in the $\zeta_\tau$s suffices to provide concentration. In
our case, we must use the randomness in $Z_\tau$ itself, which is more
delicate, since $Z_\tau$ affects the numerator $S_t$, but also the
normalizer $\Gamma_t$. In spite of this additional technical
challenge, the two self-normalized processes admit similar bounds.

Next, we turn to the action selection step, where recall that either a
single action is played deterministically, or the actions are played
uniformly at random.
\begin{lemma}
  \label{lem:selection_two_arm}
  Let $\mu_t \triangleq \EE_{a \sim \pi_t} z_{t,a}$ where $\pi_t$ is the
  solution to~\pref{eq:optimize}, and assume that the conclusion of
  \pref{lem:estimator} holds. Then with probability at least $1-\delta$
  \begin{align*}
    \reg(T) \leq \sqrt{2T\log(1/\delta)} + 2\ellipsoid \sum_{t=1}^T \sqrt{\tr(\Gamma_t^{-1}\Cov_{b \sim \pi_t}(z_{t,b}))}. 
  \end{align*}
\end{lemma}
\begin{proof}
  We first study the instantaneous regret, taking expectation over the
  random action. For this, we must consider two
  cases. First, with~\pref{lem:estimator}, if $|\Acal_t| = 1$, we
  argue that the regret is actually zero. This follows from the
  Cauchy-Schwarz inequality since assuming $\Acal_t = \{a\}$ we get
  \begin{align*}
    \langle \theta, z_{t,a} - z_{t,b}\rangle &\ge \langle \hat{\theta}_t, z_{t,a} - z_{t,b}\rangle - \ellipsoid \|z_{t,a} - z_{t,b}\|_{\Gamma_t^{-1}}
  \end{align*}
  which is non-negative using the fact that $b$ was
  eliminated. Therefore $a$ is the optimal action and we incur no
  regret. Since $\pi_t$ has no covariance, the upper bound holds.

  On the other rounds, we set $\pi_t = \textrm{Unif}(\{a,b\})$ and
  hence $\mu_t = (z_{t,a} + z_{t,b})/2$. Assuming again that $a$ is
  the optimal action, the expected regret is
\ifthenelse{\equal{\version}{arxiv}}{
  \begin{align*}
    & \langle \theta, z_{t,a} - \mu_t\rangle = \frac{1}{2}\langle \theta, z_{t,a} - z_{t,b}\rangle
    \leq \frac{1}{2}\left(\langle \hat{\theta}_t, z_{t,a} - z_{t,b}\rangle + \ellipsoid \|z_{t,a}-z_{t,b}\|_{\Gamma_t^{-1}}\right)\\
    & \leq \ellipsoid \|z_{t,a}-z_{t,b}\|_{\Gamma_t^{-1}}
    \leq 2 \ellipsoid \sqrt{\tr(\Gamma_t^{-1}\Cov_{b \sim \pi_t} (z_{t,b}))}.
  \end{align*}
}{
  \begin{align*}
    & \langle \theta, z_{t,a} - \mu_t\rangle = \frac{1}{2}\langle \theta, z_{t,a} - z_{t,b}\rangle\\
    & \leq \frac{1}{2}\left(\langle \hat{\theta}_t, z_{t,a} - z_{t,b}\rangle + \ellipsoid \|z_{t,a}-z_{t,b}\|_{\Gamma_t^{-1}}\right)\\
    & \leq \ellipsoid \|z_{t,a}-z_{t,b}\|_{\Gamma_t^{-1}}
    \leq 2 \ellipsoid \sqrt{\tr(\Gamma_t^{-1}\Cov_{b \sim \pi_t} (z_{t,b}))}.
  \end{align*}
}
  Here the first inequality uses Cauchy-Schwarz, the second
  uses~\pref{eq:filter}, since neither action was eliminated, and the
  third uses~\pref{eq:optimize}. This bounds the expected regret, and
  the lemma follows by Azuma's inequality.
\end{proof}

The last step of the proof is to control the sequence
\begin{align*}
\sum_{t=1}^T \sqrt{\tr(\Gamma_t^{-1} \Cov_{b \sim \pi_t}(z_{t,b}))}.
\end{align*}
First, recall that
$$
\Cov_{b\sim \pi_t}(z_{t,b}) \triangleq \EE_{b\sim \pi_t}\left[( z_{t,b}-\mu_t )
  (z_{t, b} -\mu_t)\tran \right]
$$
with $\mu_t \triangleq \EE_{b\sim \pi_t} [z_{t,b}]$. Since in the two-arm case
$\pi_t$ either chooses an arm deterministically or uniformly
randomizes between the two arms, the following always holds:
\[
\Cov_{b\sim \pi_t}(z_{t,b}) = (z_{t, a_t} - \mu_t)(z_{t, a_t} - \mu_t)\tran.
\]
It follows that
$\Gamma_{t+1} \triangleq \Gamma_t + \Cov_{b \sim \pi_t}(z_{t,b})$, and with
$\Gamma_1 \triangleq \lambda I$, the standard potential argument for online
ridge regression applies. We state the conclusion here, and provide a
complete proof in the appendix.

\begin{lemma}
  \label{lem:regret}
  Let $\Gamma_t$, $\pi_t$ be defined as above and define $M_t \triangleq (z_{t,a_t} - \mu_t)(z_{t,a_t}-\mu_t)\tran$. Then
  \begin{align*}
    \sum_{t=1}^T\sqrt{\tr(\Gamma_t^{-1}M_t)} \leq \sqrt{dT(1+1/\lambda)\log(1+T/(d\lambda))}.
  \end{align*}
\end{lemma}

Combining the three lemmas establishes a regret bound of
\begin{align*}
\reg(T) \leq \order\left(\sqrt{Td\log(T/\delta)\log(T/d)} +
  d\sqrt{T}\log(T/d)\right)
\end{align*}
with probability at least $1 - \delta$ in the two-action case.

\paragraph{Extending to many actions.}
Several more technical steps are required for the general
setting. First, the martingale inequality used in~\pref{lem:estimator}
requires that the random vectors are symmetric about the origin. This
is only true for the two-action case, and in fact a similar inequality
does not hold in general for the non-symmetric situation that arises
with more actions. In the non-symmetric case, both the empirical and
the population covariance must be used in the normalization, so the
analogue of~\pref{eq:two_arm_normalized} is instead
\begin{align*}
\|S_t\|_{(\Gamma_t + \EE \Gamma_t)^{-1}}^2 \leq 27d\log(1+2T/d)+54\log(4T/\delta).
\end{align*}
On the other hand, the error term for our estimator depends only on
the empirical covariance $\Gamma_t$. To correct for the discrepancy,
we use a covering argument\footnote{For technical reasons, the Matrix
  Bernstein inequality does not suffice here since it introduces a
  dependence on the maximal variance. See Appendix for details.} to
establish
\begin{align*}
\lambda I + \Gamma_t \succeq (\lambda - 6d\log(T/\delta))I + (\Gamma_t + \EE \Gamma_t)/3.
\end{align*}
With this semidefinite inequality, we can translate from the
Mahalanobis norm in the weaker self-normalized bound to one with just
$\Gamma_t$, which controls the error for the estimator.

We also argue that problem~\pref{eq:optimize} is always feasible,
which is the contents of the following lemma.
\begin{lemma}
  \label{lem:duality}
  Problem~\pref{eq:optimize} is convex and always has a feasible
  solution. Specifically, for any vectors $z_1,\ldots, z_n \in \RR^d$
  and any positive definite matrix $M$, there exists a distribution
  $w \in \Delta([n])$ with mean $\mu_w \triangleq \EE_{b \sim w}z_b$ such that
  \begin{align*}
    \forall i \in [n],  \|z_i - \mu_w\|_{M}^2 \leq \tr(M\Cov_{b \sim w}(z_b)).
  \end{align*}
\end{lemma}
The proof uses convex duality. Integrating these new arguments into
the proof for the two-action case leads to~\pref{thm:regret}.
}
\begin{figure*}[t]
\includegraphics[width=\textwidth]{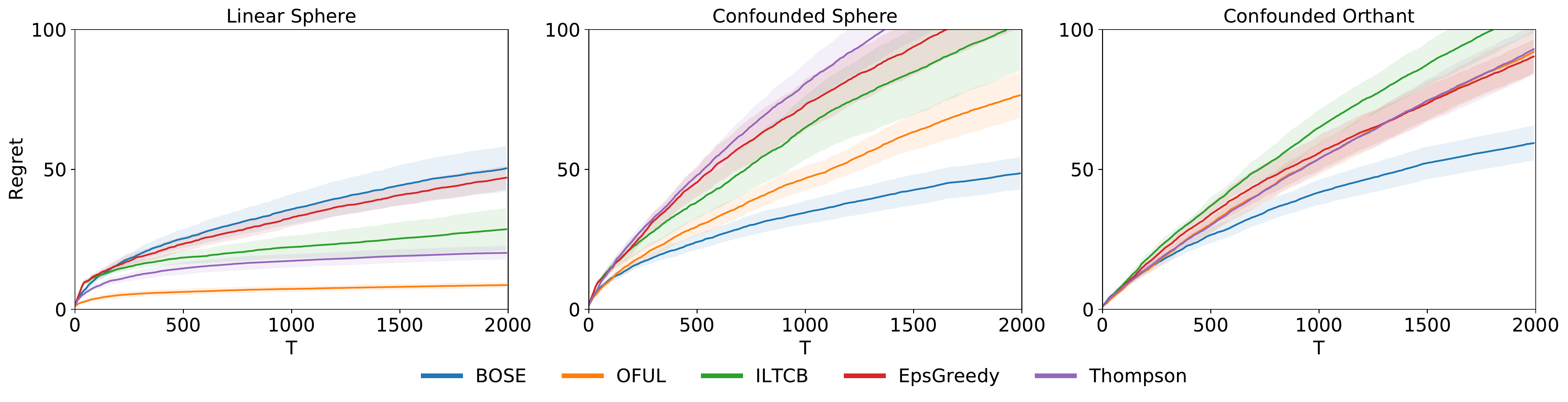} 
\vspace{-0.65cm}
\caption{Synthetic experiments with
  $d=10,K=2$. Left: A linear environment where action-features are
  uniformly from the unit sphere. Center: A confounded
  environment with features from the sphere. Right: A confounded
  environment with features from the sphere intersected with the
  positive orthant. Algorithms are \alg, \linucb
  \citep{abbasi2011improved}, \monster~\citep{agarwal2014taming},
  \greedy~\citep{langford2008epoch}, and
  \thompson~\citep{agrawal2013thompson}. Agnostic approaches use a
  linear policy class.}
\label{fig:exp}
\end{figure*}

\section{Experiments}
\label{sec:experiments}
We conduct a simple experiment to compare \alg with several other
approaches\footnote{Our code is publicly available at
  \url{http://github.com/akshaykr/oracle_cb/}.}. We simulate three
different environments that follow the semiparametric contextual
bandits model with $d=10$, $K=2$. In the first setting the reward is
linear and the action features are drawn uniformly from the unit
sphere.  In the latter two settings, we set
$f_t(x_t) = - \max_{a} \langle \theta, z_{t,a}\rangle$, which is
related to \ifthenelse{\equal{\version}{workshop}}{a lower bound
  construction}{the construction in the proof
  of~\pref{prop:det_lb}}. One of these semiparametric settings has
action features sampled from the unit sphere, while for the other, we
sample from the intersection of the unit sphere and the positive
orthant.

In~\pref{fig:exp}, we plot the performance of~\pref{alg:pseudocode}
against four baseline algorithms: (1) \linucb: the optimistic
algorithm for linear stochastic bandits~\citep{abbasi2011improved},
(2) \thompson sampling for linear contextual
bandits~\cite{agrawal2013thompson}, (3) \greedy: the $\epsilon$-greedy
approach~\citep{langford2008epoch} with a linear policy class, (4)
\monster: a more sophisticated agnostic
algorithm~\citep{agarwal2014taming} with linear policy class. The
first algorithm is deterministic, so can have linear regret in our
setting, but is the natural baseline and one we hope to
improve. Thompson Sampling is another natural baseline, and a variant
was used by~\citet{greenewald2017action} in essentially the same
setting as ours. The latter two have $(Kd)^{1/3}T^{2/3}$ and
$\sqrt{KdT}$ regret bounds respectively under our assumptions, but
require solving cost-sensitive classification problems, which are
NP-hard in general. Following prior empirical
evaluations~\citep{krishnamurthy2016contextual}, we use a surrogate
loss formulation based on square loss minimization in the
implementation. 

The results of the experiment are displayed in~\pref{fig:exp},
where we plot the cumulative regret against the number of rounds
$T$. All algorithms have a single parameter that governs the degree of
exploration. In \alg and \linucb, this is the constant $\ellipsoid$ in
the confidence bound, in \thompson it is the variance of the prior, and in \monster and \greedy it is the amount
of uniform exploration performed by the algorithm. For each algorithm
we perform 10 replicates for each of 20 values of the corresponding
parameter, and we plot the best average performance, with error bars
corresponding to $\pm 2$ standard deviations.

In the linear experiment (\pref{fig:exp}, left panel), \alg
performs the worst, but is competitive with the agnostic approaches,
demonstrating a price to pay for robustness. The experimental setup in
the center panel is identical except with confounding, and \alg is
robust to this confounding, with essentially the same performance,
while the three baselines degrade dramatically. Finally, when the
features lie in the positive orthant (right panel), \linucb degrades
further, while \alg remains highly effective.

\ifthenelse{\equal{\version}{workshop}}{}{
Regarding the baselines, we make two remarks: 
\begin{packed_enum}
\item Intuitively, the positive orthant setting is more challenging
  for \linucb since there is less inherent randomness in the
  environment to overcome the confounding effect.
\item The agnostic approaches, despite strong regret guarantees,
  perform somewhat poorly in our experiments, and we believe this for
  three reasons. First, our surrogate-loss implementation is based on
  an implicit realizability assumption, which is not satisfied
  here. Second, we expect that the constant factors in their regret
  bounds are significantly larger than those of \alg or
  \linucb. For computational reasons, we only solve the
  optimization problem in \monster every 50 rounds, which causes a
  further constant factor increase in the regret.
\end{packed_enum}
}

Overall, while \alg is worse than other approaches in the linear
environment, the experiment demonstrates that when the environment is
not perfectly linear, approaches based on realizability assumptions
(either explicitly like in \linucb, or implicitly like in 
implementations of \monster and \greedy), can fail. We emphasize that
linear environments are rare in practice, and such assumptions are
typically impossible to verify. We therefore believe that trading off
a small loss in performance in the specialized linear case for
significantly more robustness, as \alg demonstrates, is desirable.

\ifthenelse{\equal{\version}{workshop}}{}{
\section{Discussion}
This paper studies a generalization of the linear stochastic bandits
setting, where rewards are confounded by an adaptive adversary. Our
new algorithm, \alg, achieves the optimal regret, and also matches (up
to logarithmic factors) the best algorithms for the linear case. Our
empirical evaluation shows that \alg offers significantly more
robustness than prior approaches, and performs well in several
environments.  }


\newpage

\appendix
\onecolumn
\section{Using the OLS Estimator}
\label{app:bias}
Here we construct an example problem to demonstrate how using the
standard OLS estimator can fail in the semiparametric setting. While
not a comprehensive proof against all asymptotically biased
approaches, similar examples can be constructed for related
estimators.

Consider a two-dimensional problem with two actions and no stochastic
noise, where $\theta = e_2$, the second standard basis vector. On the
even rounds, the actions are $z_1 = (1,1), z_2 = (1,1/3)$ and the
confounding term is $f = -1$. On the odd rounds, the actions are
$z_1 = z_2 = (1,0)$ and the confounding term is $f = 1$. For any
policy for selecting actions, the OLS estimator before round $t$ (for
even $t$) is the solution to the following optimization problem:
\begin{align*}
\textrm{minimize}_{w\in\RR^2}\ \alpha(w_1+w_2)^2 + (1-\alpha)(w_1+w_2/3+2/3)^2 + (w_1-1)^2 = L(w)
\end{align*}
where $\alpha \in [0,1]$ corresponds to the fraction of the even
rounds (up to round $t$) where the policy chose $z_1$. We will argue
that, for any $\alpha$, the solution to this problem $\hat{w}$ has
$\hat{w}_2 < 0$. Since there is no stochastic noise, there is no need
for confidence bounds once the covariance is full rank, which happens
after the second round. Together, this implies that any sensible
policy based on $\hat{w}$ will prefer $z_2$ to $z_1$ on the even
rounds, but $z_1$ yields higher reward by a fixed constant. Thus using
OLS in a confidence-based approach leads to linear regret.

We now show that $\hat{w}_2$ is strictly negative. We have
\begin{align*}
\frac{\partial L(w)}{\partial w_1} &= 2\alpha(w_1+w_2) + 2(1-\alpha)(w_1+w_2/3+2/3) + 2(w_1-1),\\
\frac{\partial L(w)}{\partial w_2} &= 2\alpha(w_1+w_2) + \frac{2}{3}(1-\alpha)(w_1+w_2/3+2/3).
\end{align*}
Setting both equations equal to zero yields the following system:
\begin{align*}
4w_1 + (2/3+4\alpha/3)w_2 = 2/3 +4\alpha/3,\qquad (2/3 + 4\alpha/3)w_1 + (2/9 + 16\alpha/9)w_2 = 4\alpha/9 - 4/9.
\end{align*}
The solution to this system is
\begin{align*}
w_1 = \frac{(2\alpha+1)^2}{-4\alpha^2+12\alpha+1},\qquad w_2 = \frac{4\alpha^2+5}{4\alpha^2-12\alpha-1},
\end{align*}
provided that $4\alpha^2 \ne 12\alpha+1$, which is not possible with
$\alpha\in[0,1]$. In the interval $[0,1]$ we have that
$4\alpha^2 - 12\alpha-1 < 0$, and hence $w_2 < 0$. Thus, the OLS
estimator incorrectly predicts that $z_2$ receives higher reward than
$z_1$ on the even rounds. Since confidence intervals are not needed,
the algorithm suffers linear reget.

\section{Proof of~\pref{prop:det_lb}}
\label{app:lower}
We consider two possible values for the true parameter:
$\theta_1 = e_1 \in \RR^2, \theta_2 = e_2 \in \RR^2$. At all rounds,
the context $x_t = \{e_1, e_2\}$ contains just two actions, and we
further assume that the noise term $\xi_t = 0$ almost surely. Since
the action $a_t$ is a deterministic function of the history, it can
also be computed by the adaptive adversary at the beginning of the
round, and the adversary chooses
\begin{align*}
  f_t(x_t) = -\one\{a_t = \argmax_{a} \langle \theta, z_{t,a}\rangle\}.
\end{align*}
We show that $r_t(a_t) = 0$ for all rounds $t$. Assume the parameter
is $\theta_1$ so the optimal action is $a_t^\star = e_1$ and the
suboptimal action $e_2$ has $\langle \theta, e_2\rangle = 0$. If the
learner chooses action $e_2$, then the adversary sets $f_t(x_t) = 0$,
so $r_t(a_t) = 0$. On the other hand, if the learner chooses action
$e_1$, then the adversary sets $f_t(x_t) = -1$ so the reward is also
zero. Similarly, if $\theta = \theta_2$, the observed reward is always
zero. Since the algorithm is deterministic, it behaves identically
regardless of whether the parameter is $\theta_1$ or $\theta_2$. In
one of these instances the algorithm must choose the suboptimal action
at least $T/2$ times, leading to the lower bound.

\section{Proof for the Two-Action Case}
We first focus on the simpler two action case.  Before turning to the
main analysis, we prove two supporting lemmas. The first is an
algebraic inequality relating matrix determinants to traces. This
inequality also appears in~\citet{abbasi2011improved}.
\begin{lemma}
\label{lem:det_to_trace}
Let $X_1,\ldots,X_n$ denote vectors in $\RR^d$ with $\|X_i\|_2\le L$ for
all $i \in [n]$. Define $\Gamma \triangleq \lambda I + \sum_{i=1}^n X_iX_i\tran$. Then
\begin{align*}
  \det(\Gamma) \leq (\lambda + nL^2/d)^d.
\end{align*}
\end{lemma}
\begin{proof}
  We will apply the following standard argument:
  \begin{align*}
    \det(\Gamma)^{1/d} \leq \frac{1}{d}\tr(\Gamma) = \frac{1}{d}\tr(\lambda I) + \frac{1}{d}\sum_{i=1}^n\tr(X_iX_i\tran) = \lambda + \frac{1}{d} \sum_{i=1}^n \|X_i\|_2^2 \leq \lambda + nL^2/d.
  \end{align*}
  The first step is a spectral version of the AM-GM inequality and the
  remaining steps use linearity of the trace operator and the
  boundedness conditions.
\end{proof}

The second lemma is a new self-normalized concentration inequality for
vector valued martingales.
\begin{lemma}[Symmetric self-normalized inequality]
  \label{lem:self_normalized_symmetric}
  Let $\{\Fcal_t\}_{t=1}^T$ be a filtration and let
  $\{(Z_t,\zeta_t)\}_{t=1}^T$ be a stochastic process with
  $Z_t \in \RR^d$ and $\zeta_t \in \RR$ such that (1) $(Z_t,\zeta_t)$
  is $\Fcal_t$ measurable, (2) $|\zeta_t| \le M$ for all $t \in [T]$,
  (3) $Z_t \independent \zeta_t | \Fcal_t$, (4)
  $\EE[Z_t | \Fcal_t]=0$, and (5) for all $x \in \RR^d$,
  $\Lcal(\langle x,Z_t\rangle \mid \Fcal_t) = \Lcal(-\langle
  x,Z_t\rangle \mid \Fcal_t)$ where $\Lcal$ denotes the probability
  law, so that $Z_t$ is conditionally symmetric. Let
  $\Sigma \triangleq \sum_{t=1}^T Z_t Z_t\tran$. Then for any positive definite
  matrix $Q$ we have
  \begin{align*}
    \PP\left[ \left\|\sum_{t=1}^TZ_t\zeta_t\right\|_{(Q + M^2\Sigma)^{-1}}^2 \ge 2 \log\left(\frac{1}{\delta}\sqrt{\frac{\det(Q+M^2\Sigma)}{\det(Q)}}\right)\right] \le \delta.
  \end{align*}
\end{lemma}
\begin{proof}
  The proof follows the recipe in~\citet{delapena2009theory} (See
  also~\citet{delapena2008self} for a more comprehensive treatment
  including the univariate case). We start by applying the Chernoff
  method. Let $\bar{\Sigma} \triangleq Q + M^2\Sigma$. We can write
  \begin{align*}
    \PP\left[\left\|\sum_{t=1}^T Z_t \zeta_t\right\|_{\bar{\Sigma}^{-1}}^2 \ge 2 \log\left(\frac{1}{\delta}\sqrt{\frac{\det(\bar{\Sigma})}{\det(Q)}}\right)\right]
    & = \PP\left[\exp\left(\frac{1}{2}\left\|\sum_{t=1}^T Z_t\zeta_t\right\|_{\bar{\Sigma}^{-1}}^2\right) \ge \frac{1}{\delta}\sqrt{\frac{\det(\bar{\Sigma})}{\det(Q)}} \right]\\
    & \leq \delta \EE\left[\sqrt{\frac{\det(Q)}{\det(\bar{\Sigma})}} \exp\left(\frac{1}{2}\left\|\sum_{t=1}^T Z_t\zeta_t\right\|_{\bar{\Sigma}^{-1}}^2\right)\right].
  \end{align*}
  Therefore, if we prove that this latter expectation is at most one,
  we will arrive at the conclusion. A similar statement appears in
  Theorem 1 of~\citet{delapena2009theory}, but our process is slightly
  different due to the presence of $\zeta_t$.  To bound this latter
  expectation, fix some $\lambda \in \RR^d$ and consider an 
  exponentiated process with the increments
  \begin{align*}
    D_t^\lambda \triangleq \exp\left(\langle \lambda, Z_t\zeta_t\rangle - \frac{M^2 \langle \lambda, Z_t\rangle^2}{2}\right).
  \end{align*}
  Observe that $\EE [D_t^\lambda | \Fcal_t] \leq 1$ since by the conditional symmetry of $Z_t$, we have
  \begin{align*}
    \EE [D_t^\lambda | \Fcal_t] &= \EE \left[\EE\left[ D_t^\lambda  \mid \Fcal_t, \zeta_t\right] \mid \Fcal_t\right]\\
    & = \EE\left[ \EE\left[\exp\left(\frac{-M^2\langle \lambda, Z_t\rangle^2}{2}\right) \times \frac{1}{2}\left(\exp(\langle \lambda, Z_t\zeta_t\rangle) + \exp(- \langle \lambda, Z_t\zeta_t\rangle \right) \mid \Fcal_{t}, \zeta_t\right]\mid \Fcal_t\right]\\
    & = \EE\left[\EE\left[\exp\left(\frac{-M^2\langle \lambda, Z_t\rangle^2}{2}\right) \times \cosh(\langle \lambda, Z_t\zeta_t) \mid \Fcal_t, \zeta_t\right] \mid \Fcal_t \right]\\
    & \leq \EE\left[\EE\left[\exp\left(\frac{-M^2\langle \lambda, Z_t\rangle^2}{2} + \frac{\langle \lambda, Z_t\zeta_t\rangle^2}{2}\right)\mid \Fcal_t, \zeta_t\right] \mid \Fcal_t\right] \leq 1.
  \end{align*}
  This argument first uses the conditional symmetry of $Z_t$ and the
  conditional independence of $Z_t,\zeta_t$, then the identity
  $(e^x + e^{-x})/2 = \cosh(x)$ and finally the analytical inequality
  $\cosh(x) \leq e^{x^2/2}$. Finally in the last step we use the bound
  $|\zeta_t| \leq M$. This implies that the martingale
  $U_t^\lambda \triangleq \prod_{\tau=1}^t D_\tau^\lambda$ is a super-martingale
  with $\EE[ U_t^\lambda ]\leq 1$ for all $t$, since by induction
  \begin{align}
    \EE[U_t^\lambda] = \EE[U_{t-1}^\lambda \EE[D_t^\lambda | \Fcal_t]]\leq \EE[U_{t-1}^\lambda]\leq \ldots \leq 1.\label{eq:two_arm_exponentiated}
  \end{align}
  Now we apply the method of mixtures. In a standard application of
  the Chernoff method, we would choose $\lambda$ to maximize
  $\EE[U_T^\lambda]$, but since we still have an expectation, we
  cannot swap expectation and maximum. Instead, we integrate the
  inequality $\EE[U_T^\lambda] \leq 1$, which holds for any $\lambda$,
  against $\lambda$ drawn from a Gaussian distribution with covariance
  $Q^{-1}$. By Fubini's theorem, we can swap the expectations to
  obtain
  \begin{align*}
    1 &\ge \EE_{\lambda \sim \Ncal(0,Q^{-1})} \EE[U_T^\lambda] = \EE \int U_T^\lambda (2\pi)^{-d/2}\sqrt{\det(Q)} \exp(-\lambda\tran Q\lambda /2)d\lambda\\
    & = \EE \int (2\pi)^{-d/2}\sqrt{\det(Q)}\exp\left(\sum_{t=1}^T\langle \lambda, Z_t\zeta_t\rangle - \frac{M^2 \lambda\tran(\sum_{t=1}^TZ_tZ_t\tran)\lambda + \lambda\tran Q\lambda}{2}\right)d\lambda\\
    & = \EE \int (2\pi)^{-d/2}\sqrt{\det(Q)}\exp\left(\langle \lambda, S\rangle - \frac{M^2 \lambda\tran\Sigma\lambda + \lambda\tran Q\lambda}{2}\right)d\lambda,
  \end{align*}
  where $S \triangleq \sum_{t=1}^TZ_t\zeta_t$ and recall that $\Sigma \triangleq \sum_{t=1}^TZ_tZ_t\tran$. By completing the square, the term in the exponent can be rewritten as
  \begin{align*}
    \langle \lambda, S\rangle - \frac{M^2\lambda\tran\Sigma\lambda + \lambda\tran Q\lambda}{2} = \frac{1}{2}\left( -(\lambda - \bar{\Sigma}^{-1}S)\tran \bar{\Sigma}(\lambda - \bar{\Sigma}^{-1}S) + S\tran \bar{\Sigma}^{-1}S\right),
  \end{align*}
  where recall that $\bar{\Sigma} \triangleq M^2\Sigma + Q$. As such we obtain
  \begin{align*}
    1 & \geq \EE\left[\exp\left( S\tran \bar{\Sigma}^{-1}S/2\right) \times \int (2\pi)^{-d/2}\sqrt{\det(Q)} \exp\left(\frac{ - (\lambda - \bar{\Sigma}^{-1}S)\tran \bar{\Sigma}(\lambda - \bar{\Sigma}^{-1}S)}{2}\right)\right]d\lambda \\
    & = \EE\sqrt{\frac{\det(Q)}{\det(\bar{\Sigma})}} \exp\left(S\tran \bar{\Sigma}^{-1}S\right).
  \end{align*}
This proves the lemma.
\end{proof}

Equipped with the two lemmas, we can now turn to the analysis of the
influence-adjusted estimator.
\begin{lemma}[Restatement of~\pref{lem:estimator}]
  Under~\pref{as:env} and~\pref{as:bd}, with probability at
  least $1-\delta$, the following holds simultaneously for all
  $t \in [T]$:
  \begin{align*}
    \|\hat{\theta}_t - \theta\|_{\Gamma_t} \leq \sqrt{\lambda} + \sqrt{9 d \log(1+ T/(d\lambda)) +
    18\log(T/\delta)}.
  \end{align*}
\end{lemma}
\begin{proof}
  Recall that we define $\hat{\theta}_t,\Gamma_t$ to be the estimator
  and matrix used in round $t$, based on $t-1$ examples. Fixing a
  round $t$, we start by expanding the definition of
  $\hat{\theta}_t$. We use the shorthand $z_\tau \triangleq z_{\tau,a_\tau}$,
  $\mu_\tau \triangleq \EE_{b \sim \pi_\tau} \left[ z_{\tau,b} \right]$, and
  $r_\tau \triangleq r_\tau(a_\tau)$.
  \begin{align*}
    \hat{\theta}_t & = \Gamma_t^{-1}\sum_{\tau=1}^{t-1}(z_\tau - \mu_\tau)r_\tau = \Gamma_t^{-1} \sum_{\tau=1}^{t-1} (z_\tau - \mu_\tau) (\langle \theta, z_\tau\rangle + f_\tau(x_\tau) + \xi_\tau)\\
    & = \Gamma_t^{-1} \sum_{\tau=1}^{t-1} (z_\tau - \mu_\tau) (\langle \theta, z_\tau -\mu_\tau\rangle + \langle \theta, \mu_\tau\rangle + f_\tau(x_\tau) + \xi_\tau)\\
    & = (\Gamma_t)^{-1} (\Gamma_t - \lambda I)\theta + \Gamma_t^{-1}\sum_{\tau=1}^{t-1} (z_\tau - \mu_\tau)(\langle \theta,\mu_\tau\rangle + f_\tau(x_\tau) + \xi_\tau).
  \end{align*}
  Let $Z_\tau \triangleq z_\tau - \mu_\tau$ and
  $\zeta_\tau \triangleq \langle \theta,\mu_\tau\rangle + f_\tau(x_\tau) + \xi_\tau$. With this expansion, we can write
  \begin{align*}
    \|\hat{\theta}_t - \theta\|_{\Gamma_t} &= \|-\lambda \Gamma_t^{-1}\theta + \Gamma_t^{-1}\sum_{\tau=1}^{t-1}Z_\tau\zeta_\tau\|_{\Gamma_t}
                                            \leq \|\lambda \theta\|_{\Gamma_t^{-1}} + \left\|\sum_{\tau=1}^{t-1}Z_\tau\zeta_\tau\right\|_{\Gamma_t^{-1}}
                                            \leq \sqrt{\lambda} + \left\|\sum_{\tau=1}^{t-1}Z_\tau\zeta_\tau\right\|_{\Gamma_t^{-1}}.
  \end{align*}
  To finish the proof, we apply~\pref{lem:self_normalized_symmetric}
  to this last term. To verify the preconditions of the lemma, let
  $\Fcal_\tau \triangleq
  \sigma(x_1,\ldots,x_\tau,a_1,\ldots,a_{\tau-1},\xi_1,\ldots,\xi_{\tau-1})$
  denote the $\sigma$-algebra corresponding to the $\tau^\textrm{th}$
  round, after observing the context $x_\tau$. Then the policy
  $\pi_\tau$ and hence the action $a_\tau$ are $\Fcal_\tau$ measurable
  and so is the noise term $\xi_\tau$. Therefore,
  $Z_\tau = z_{\tau,a_\tau} - \EE_{a \sim
    \pi_\tau}\left[z_{\tau,a}\right]$ is measurable, which verifies
  the first precondition. Using the boundedness properties
  in~\pref{as:bd}, we know that $|\zeta_\tau| \leq 3 \triangleq M$,
  and by construction of the random variables, we have
  $Z_\tau \independent \zeta_\tau | \Fcal_{\tau}$ and
  $\EE \left[Z_\tau | \Fcal_\tau\right] = 0$. Finally, for the
  symmetry property, either $Z_\tau|\Fcal_\tau \equiv 0$ if one action
  is eliminated, or otherwise we have
  $\mu_\tau = \frac{1}{2}(z_{\tau,1} + z_{\tau,2})$ since there are
  only two actions. In this case the random variable
  $Z_\tau |\Fcal_\tau = \epsilon_\tau(z_{\tau,1} - z_{\tau_2})/2$
  where $\epsilon_\tau$ is a Rademacher random variable. By inspection
  this is clearly conditionally symmetric. As such, we may
  apply~\pref{lem:self_normalized_symmetric}, which reveals that with
  probability at least $1-\delta$,
  \begin{align*}
   \left \|\sum_{\tau=1}^{t-1}Z_\tau \zeta_\tau\right\|_{\Gamma_t^{-1}}^2 &= M^2\left\|\sum_{\tau=1}^{t-1} Z_\tau\zeta_\tau\right\|_{(M^2\Gamma_t)^{-1}}^2 \leq 2M^2\log\left(\frac{1}{\delta}\sqrt{\frac{\det(M^2\Gamma_t)}{\det(M^2\lambda I)}}\right)\\
    &= 18\log\left(\sqrt{\lambda^{-d}\det(\Gamma_t)}/\delta\right).
  \end{align*}
  The inequality here is~\pref{lem:self_normalized_symmetric}
  with $Q = M^2\lambda I$, and for the last equality we use that
  $\det(cQ) = c^d\det(Q)$ for a $d \times d$ positive semidefinite
  matrix $Q$. As two final steps, we apply~\pref{lem:det_to_trace} and take a union bound over all rounds
  $T$. Combining these, we get that for all $T$,
  \begin{align*}
    \|\hat{\theta}_t - \theta\|_{\Gamma_t} & \leq \sqrt{\lambda} + \left\|\sum_{\tau=1}^{t-1}Z_\tau\zeta_\tau\right\|_{\Gamma_t^{-1}} \leq  \sqrt{\lambda} + \sqrt{18 \left(\log(\sqrt{\lambda^{-d}\det(\Gamma_t)}) + \log(T/\delta)\right)}\\
    & \leq \sqrt{\lambda} + \sqrt{9 d \log(1+ T/(d\lambda)) + 18\log(T/\delta)}. \tag*\qedhere
  \end{align*}
\end{proof}

Therefore, with
$\ellipsoid \triangleq \sqrt{\lambda} + \sqrt{9d\log(1+T/(d\lambda)) +
  18\log(T/\delta)}$ we can apply~\pref{lem:selection_two_arm} to bound the regret by
\begin{align*}
  \reg(T) \leq \sqrt{2T\log(1/\delta)} + 2\ellipsoid \sum_{t=1}^T \sqrt{\tr(\Gamma_t^{-1} \Cov_{b \sim \pi_t}(z_{t,b}))}.
\end{align*}
Via a union bound, this inequality holds with probability at least
$1-2\delta$. To finish the proof we need to analyze this latter
term. This is the contents of the following lemma. A related
statement, with a similar proof, appears
in~\citet{abbasi2011improved}.
\begin{lemma}
  Let $X_1,\ldots,X_T$ be a sequence of vectors in $\RR^d$ with
  $\|X_t\|_2 \le 1$ and define $\Gamma_1 \triangleq \lambda I$,
  $\Gamma_t \triangleq \Gamma_{t-1} + X_{t-1}X_{t-1}\tran $. Then
  \begin{align*}
    \sum_{t=1}^T \sqrt{\tr(\Gamma_t^{-1} X_tX_t\tran )} \le \sqrt{Td (1+1/\lambda)\log(1+T/(d\lambda))}.
  \end{align*}
\end{lemma}
\begin{proof}
  First, apply the Cauchy-Schwarz inequality to the left hand side to obtain
  \begin{align*}
    \sum_{t=1}^T \sqrt{\tr(\Gamma_t^{-1}X_tX_t\tran )} \leq \sqrt{T} \sqrt{\sum_{T=1}^T \tr(\Gamma_t^{-1} X_t X_t\tran )}.
  \end{align*}
  For the remainder of the proof we work only with the second
  term. Let us start by analyzing a slightly different quantity,
  $\tr(\Gamma_{t+1}^{-1} X_tX_t\tran )$. By concavity of $\log \det(M)$, we have
  \begin{align*}
    \log \det (\Gamma_t) \leq \log \det (\Gamma_{t+1}) + \tr(\Gamma_{t+1}^{-1}(\Gamma_t - \Gamma_{t+1})),
  \end{align*}
  which implies
  \begin{align*}
    \tr(\Gamma_{t+1}^{-1} X_tX_t\tran ) = \tr(\Gamma_{t+1}^{-1}(\Gamma_{t+1} - \Gamma_t)) \leq \log \det (\Gamma_{t+1}) - \log \det(\Gamma_t)
  \end{align*}
  As such, we obtain a telescoping sum
  \begin{align*}
    \sum_{t=1}^T \tr(\Gamma_{t+1}^{-1} X_tX_t\tran ) \leq \log \det(\Gamma_{T+1}) - \log \det( \Gamma_1) \leq d \log (\lambda + T/d) - d \log \lambda = d \log(1 +  T/(d\lambda))
  \end{align*}
  The first inequality here uses the concavity argument and the second
  uses~\pref{lem:det_to_trace}. To finish the proof, we must translate
  back to $\Gamma_t^{-1}$. For this, we use the
  Sherman-Morrison-Woodbury identity, which reveals that
  \begin{align*}
    X_t\tran \Gamma_{t+1}^{-1}X_t &= X_t\tran  (\Gamma_t + X_tX_t\tran )^{-1} X_t = X_t\tran \left(\Gamma_t^{-1} - \frac{\Gamma_t^{-1}X_tX_t\tran \Gamma_t^{-1}}{1 + \|X_t\|_{\Gamma_t^{-1}}^2}\right)X_t\\
    &= \frac{\|X_t\|_{\Gamma_t^{-1}}^2}{1 + \|X_t\|_{\Gamma_t^{-1}}^2} \ge (1+1/\lambda)^{-1} \|X_t\|_{\Gamma_t^{-1}}^2.
  \end{align*}
  Here in the last step we use that
  $\|X_t\|_{\Gamma_t^{-1}}^2 \leq \|X_t\|_{(\lambda I)^{-1}}^2 \leq
  1/\lambda$. Overall, we obtain
  \begin{align*}
    \sum_{t=1}^T\tr(\Gamma_{t}^{-1}X_tX_t\tran ) \leq (1+1/\lambda) d \log(1+T/(d\lambda)),
  \end{align*}
  and combined with the first application of Cauchy-Schwarz, this proves the lemma. 
\end{proof}

Combining the lemmas, we have that with probability at least
$1-2\delta$, the regret is at most
\begin{align*}
& \reg(T)\leq \sqrt{2T\log(1/\delta)} + 2\ellipsoid \sqrt{Td(1+1/\lambda) \log(1+T/(d\lambda))}\\
     & = \sqrt{2T\log(1/\delta)} + 2\sqrt{Td(1+1/\lambda)\log(1+T/(d\lambda))}\left(\sqrt{\lambda} + \sqrt{9d\log(1+T/(d\lambda)) + 18\log(T/\delta)}\right).
\end{align*}
With $\lambda = 1$, this bound is
$\order\left(\sqrt{Td\log(T/\delta)\log(T/d)} +
  d\sqrt{T}\log(T/d)\right)$.

\section{Proof for the General Case}
We now turn to the more general case. We need several additional lemmas.

\begin{lemma}[Restatement of~\pref{lem:duality}]
  Problem~\pref{eq:optimize} is convex and always has a feasible
  solution. Specifically, for any vectors $z_1,\ldots, z_n \in \RR^d$
  and any positive definite matrix $M$, there exists a distribution
  $w \in \Delta([n])$ with mean $\mu_w = \EE_{b \sim w}[z_b]$ such
  that
  \begin{align*}
    \forall i \in [n], \quad \|z_i - \mu_w\|_{M}^2 \leq \tr(M\Cov_{b \sim w}(z_b)).
  \end{align*}
\end{lemma}
\begin{proof}
  We analyze the minimax program
  \begin{align*}
    \min_{w \in \Delta([n])} \max_{i \in [n]} \|z_i - \mu_w\|_M^2 - \tr(M\Cov_w(z)).
  \end{align*}
  The goal is to show that the value of this program is non-negative,
  which will prove the result. Expanding the definitions, we have
  \begin{align*}
    & \min_{w \in \Delta([n])} \max_{i \in [n]} \|z_i - \mu_w\|_M^2 - \tr(M\Cov_w(z))\\
    & = \min_{w \in \Delta([n])} \max_{v \in \Delta([n])} \sum_{i} v_i \left(\|z_i - \mu_w\|_M^2 +\mu_w\tran M\mu_w - \sum_j w_j z_j\tran Mz_j\right)\\
    & = \min_{v \in \Delta([n])} \max_{w \in \Delta([n])} \sum_{i} v_i \left(\|z_i - \mu_w\|_M^2 +\mu_w\tran M\mu_w - \sum_j w_j z_j\tran Mz_j\right).
  \end{align*}
  The last equivalence here is Sion's Minimax
  Theorem~\citep{sion1958general}, which is justified since both
  domains are compact convex subsets of $\RR^n$ and since the
  objective is linear in the maximizing variable $v$, and convex in
  the minimizing variable $w$. This convexity is clear since $\mu_w$
  is a linear in $w$, and hence the first two terms are convex
  quadratics (since $M$ is positive definite), while the third term is
  linear in $w$. Thus Sion's theorem lets us swap the order of the
  minimization and maximization.

  Now we upper bound the solution by setting $w = v$. This gives
  \begin{align*}
    & \leq \max_{v \in \Delta([n])} \sum_i v_i\left( \| z_i- \mu_v\|_M^2 + \mu_v\tran M\mu_v - \sum_j v_j z_j\tran Mz_j\right)\\
    & = \max_{v \in \Delta([n])} \sum_i v_i\left( (z_i - \mu_v)\tran M (z_i - \mu_v) + \mu_v\tran M\mu_v - \sum_j v_j z_j\tran Mz_j \right) = 0. \tag*\qedhere
  \end{align*}
\end{proof}

To prove the analog of~\pref{lem:self_normalized_symmetric}, we
need several additional tools. First, we use Freedman's inequality to
derive a positive-semidefinite inequality relating the sample covariance
matrix to the population matrix.
\begin{lemma}
  \label{lem:matrix_bernstein}
  Let $X_1,\ldots,X_n$ be conditionally centered random vectors in
  $\RR^d$ adapted to a filtration $\{\Fcal_t\}_{t=1}^n$ with
  $\|X_i\|_2 \le 1$ almost surely. Define
  $\hat{\Sigma} \triangleq \sum_{i=1}^n X_i X_i\tran $ and
  $\Sigma \triangleq \sum_{i=1}^n \EE[X_iX_i\tran  \mid \Fcal_i]$. Then, with
  probability at least $1-\delta$, the following holds simultaneously
  for all unit vectors $v \in \RR^d$:
  \begin{align*}
    v\tran \Sigma v 
\leq 2v\tran \hat{\Sigma}v + 9d\log(9n) + 8\log(2/\delta).
  \end{align*}
\end{lemma}
This lemma is related to the Matrix Bernstein inequality, which can be
used to control $\|\Sigma - \hat{\Sigma}\|_2$, a quantity that is
quite similar to what we are controlling here. The Matrix Bernstein
inequality can be used to derive a high probability bound of the form
\begin{align*}
  \forall v \in \RR^d, \|v\|_2 = 1, \quad v\tran (\Sigma - \hat{\Sigma})v \leq  \frac{1}{2}\|\Sigma\|_2 + c \log(dn/\delta),
\end{align*}
for a constant $c > 0$. On one hand, this bound is stronger than ours
since the deviation term depends only logarithmically on the
dimension. However, the variance term involves the spectral norm
rather than a quantity that depends on $v$ as in our bound. Thus,
Matrix Bernstein is worse when $\Sigma$ is highly ill-conditioned, and
since we have essentially no guarantees on the spectrum of $\Sigma$,
our specialized inequality, which is more adaptive to the specific
direction $v$, is crucial.
Moreover, the worse dependence on $d$ is inconsequential, since the
error will only appear in a lower order term.

\begin{proof}
  First consider a single unit vector $v \in \RR^d$, we will apply a
  covering argument at the end of the proof. By assumption,
  the sequence of sums
    $\{\sum_{i=1}^\tau v\tran (X_i X_i\tran - \EE[X_i X_i\tran \mid \Fcal_i]) v\}_{\tau=1}^n$ is a
    martingale, so we may apply Freedman's
  inequality~\citep{freedman1975tail,beygelzimer2011contextual}, which
  states that with probability at least $1-\delta$
  \begin{align*}
    |v\tran (\hat{\Sigma} - \Sigma) v| \leq 2 \sqrt{\sum_{i=1}^n\Var(v\tran (X_iX_i\tran  - \EE [X_iX_i\tran \mid\Fcal_i])v\mid \Fcal_i) \log(2/\delta)} + 2 \log(2/\delta).
  \end{align*}
  Let us now upper bound the variance term: for each $i$,
  \begin{align*}
    \Var(v\tran (X_iX_i\tran  - \EE [X_iX_i\tran  \mid \Fcal_i])v \mid \Fcal_i) &\leq \EE[(v\tran (X_iX_i\tran  - \EE [X_iX_i\tran \mid \Fcal_i]  \mid\Fcal_i)v)^2 \mid \Fcal_i]\\
    & \leq \EE[(v\tran X_i)^4 \mid \Fcal_i] \leq v\tran  \EE[ X_iX_i\tran  \mid \Fcal_i] v,
  \end{align*}
  where the last inequality follows from the fact that
    $\|X_i\|_2\leq 1$ and $\|v\|_2\leq 1$.  Therefore, the cumulative
  conditional variance is at most $v\tran \Sigma v$. Plugging this
  into Freedman's inequality gives us
  \begin{align*}
    |v\tran (\hat{\Sigma} - \Sigma) v| \leq 2 \sqrt{v\tran \Sigma v \log(2/\delta)} + 2 \log(2/\delta).
  \end{align*}
  Now, using the fact that $2\sqrt{ab} \leq \alpha a + b/\alpha$ for any $\alpha > 0$, with the choice $\alpha = 1/2$, we get
  \begin{align*}
    |v\tran (\hat{\Sigma} - \Sigma)v | \leq v\tran \Sigma v/2 + 4 \log(2/\delta).
  \end{align*}
  Re-arranging, this implies
  \begin{align}
    \label{eq:psd_concentration_single}
    v\tran \Sigma v \leq 2 v\tran \hat{\Sigma}v + 8 \log(2/\delta),
  \end{align}
  which is what we would like to prove, but we need it to hold simultaneously for all unit vectors $v$.

  To do so, we now apply a covering argument. Let $N$ be an
  $\epsilon$-covering of the unit sphere in the projection
  pseudo-metric $d(u,v) = \|uu\tran - vv\tran \|_2$, with covering
  number $\Ncal(\epsilon)$. Then via a union bound, a version
  of~\pref{eq:psd_concentration_single} holds simultaneously for all
  $v \in N$, where we rescale $\delta \to \delta/\Ncal(\epsilon)$.

  Consider another unit vector $u$ and let $v$ be the covering element. We have
  \begin{align*}
    u\tran \Sigma u &= \tr(\Sigma(uu\tran  - vv\tran )) + v\tran \Sigma v \leq \tr(\Sigma(uu\tran  - vv\tran )) + 2v\tran \hat{\Sigma}v + 8 \log(2\Ncal(\epsilon)/\delta)\\
    & = \tr((\Sigma - 2\hat{\Sigma})(uu\tran  - vv\tran )) + 2u\tran \hat{\Sigma}u + 8\log(2\Ncal(\epsilon)/\delta)\\
    & \leq \|\Sigma - 2\hat{\Sigma}\|_{\star} \epsilon + 2u\tran \hat{\Sigma}u + 8\log(2\Ncal(\epsilon)/\delta).
  \end{align*}
  Here $\|\cdot\|_{\star}$ denotes the nuclear norm, which is dual to
  the spectral norm $\|\cdot\|_2$. Since all vectors are bounded by
  $1$, we obtain
  \begin{align*}
    \|\Sigma - 2\hat{\Sigma}\|_{\star} \leq d \lambda_{\max}(\Sigma - 2\hat{\Sigma}) \leq 3dn.
  \end{align*}
  Overall, the following bound holds simultaneously for all unit vectors $v \in \RR^d$, except with probability at most $\delta$:
  \begin{align*}
    v\tran \Sigma v \leq 3dn\epsilon + 2v\tran \hat{\Sigma}v + 8\log(2\Ncal(\epsilon)/\delta).
  \end{align*}

  The last step of the proof is to bound the covering number
  $\Ncal(\epsilon)$. For this, we argue that a covering of the unit
  sphere in the Euclidean norm suffices, and by standard volumetric
  arguments, this set has covering number at most $(3/\epsilon)^d$. To
  see why this suffices, let $u$ be a unit vector and let $v$ be the
  covering element in the Euclidean norm, which implies that
  $\|u - v\|_2 \le \epsilon$. Further assume that
  $\langle u,v \rangle > 0$, which imposes no restriction since the
  projection pseudo-metric is invariant to multiplying by $-1$. By
  definition we also have $\langle u,v\rangle \leq 1$. Note that the
  projection norm is equivalent to the sine of the principal angle
  between the two subspaces, which once we restrict to vectors with
  non-negative inner product means that
  $\|uu\tran  - vv\tran \|_2 = \sin \angle(u,v)$. Now
  \begin{align*}
    \sin \angle(u,v) = \sqrt{1 - \langle u, v\rangle^2} & = \sqrt{(1+\langle u,v\rangle)(1-\langle u, v\rangle)} \\
    & \leq \sqrt{2(1 - \langle u, v\rangle)} = \sqrt{\|u\|_2^2 + \|v\|_2^2 - 2\langle u, v\rangle} = \|u-v\|_2 \le \epsilon.
  \end{align*}
  Using the standard covering number bound, we now have
  \begin{align*}
    v\tran \Sigma v \leq 3dn\epsilon + 2v\tran \hat{\Sigma}v + 8d\log(3/\epsilon) + 8\log(2/\delta).
  \end{align*}
  Setting $\epsilon=1/(3n)$ gives
  \begin{align*}
    v\tran \Sigma v \leq d + 2v\tran \hat{\Sigma}v + 8d\log(9n) + 8\log(2/\delta) \leq 2v\tran \hat{\Sigma}v + 9d\log(9n) + 8\log(2/\delta). \tag*\qedhere
  \end{align*}
\end{proof}

With the positive semidefinite inequality, we can work towards a
self-normalized martingale concentration bound. The following is a
restatement of Lemma 7 from~\citet{delapena2009theory}.
\begin{lemma}[Lemma 7 of~\citet{delapena2009theory}]
\label{lem:exponentiated_asymmetric}
Let $\{X_i\}_{i=1}^n$ be a sequence of conditionally centered
vector-valued random variables adapted to the filtration
$\{\Fcal_{i}\}_{i=1}^n$ and such that $\|X_i\|_2 \le B$ for some
constant $B$. Then
\begin{align*}
U_n(\lambda) = \exp\left(\lambda\tran \sum_{i=1}^nX_i - \lambda\tran \left(\sum_{i=1}^nX_iX_i\tran  + \EE[X_iX_i\tran  | \Fcal_i]\right)\lambda/2\right)
\end{align*}
is a supermartingale with $\EE[U_n(\lambda)] \le 1$ for all $\lambda \in \RR^d$.
\end{lemma}
The lemma is related to~\pref{eq:two_arm_exponentiated}, but does not
require that conditional probability law for $X_i$ is symmetric, which
we used previously. To remove the symmetry requirement, it is crucial
that the quadratic self-normalization has both empirical and
population terms. With this lemma, the same argument as in the proof
of~\pref{lem:self_normalized_symmetric}, yields a self-normalized tail
bound.
\begin{lemma}
  \label{lem:self_normalized}
  Let $\{\Fcal_t\}_{t=1}^T$ be a filtration and let
  $\{(Z_t,\zeta_t)\}_{t=1}^T$ be a stochastic process with
  $Z_t \in \RR^d$ and $\zeta_t \in \RR$ such that (1) $(Z_t,\zeta_t)$
  is $\Fcal_t$ measurable, (2) $|\zeta_t| \le M$ for all $t \in [T]$,
  (3) $Z_t \independent \zeta_t | \Fcal_t$, and (4)
  $\EE[Z_t | \Fcal_t]=0$. Let $\hat{\Sigma} \triangleq \sum_{t=1}^T Z_t Z_t\tran $
  and $\Sigma \triangleq \sum_{t=1}^T \EE[ Z_tZ_T\tran  | \Fcal_t]$. Then for any
  positive definite matrix $Q$ we have
  \begin{align*}
    \PP\left[ \left\|\sum_{t=1}^TZ_t\zeta_t\right\|_{(Q + M^2(\hat{\Sigma}+\Sigma))^{-1}}^2 \ge 2 \log\left(\frac{1}{\delta}\sqrt{\frac{\det(Q+M^2(\hat{\Sigma}+\Sigma))}{\det(Q)}}\right)\right] \le \delta.
  \end{align*}
\end{lemma}
\begin{proof}
The proof is identical to~\pref{lem:self_normalized_symmetric},
but uses~\pref{lem:exponentiated_asymmetric} in lieu of~\pref{eq:two_arm_exponentiated}.
\end{proof}

We can now analyze the influence-adjusted estimator.
\begin{lemma}
  \label{lem:estimator_general}
  Under~\pref{as:env} and~\pref{as:bd} and assuming that
  $\lambda \ge 4d\log(9T) + 8\log(4T/\delta)$, with probability at
  least $1-\delta$, the following holds simultaneously for all $t \in
  [T]$: \begin{align*}
\|\hat{\theta}_t - \theta\|_{\Gamma_t} \leq \sqrt{\lambda} + \sqrt{27d\log(1 + 2T/d) + 54\log(4T/\delta)}.
  \end{align*}
\end{lemma}
\begin{proof}
Using the same argument as in the proof of~\pref{lem:estimator}, we get
\begin{align*}
  \|\hat{\theta}_t - \theta\|_{\Gamma_t} \leq \sqrt{\lambda} + \left\|\sum_{\tau=1}^{t-1} Z_\tau\zeta_\tau\right\|_{\Gamma_t^{-1}},
\end{align*}
where $Z_\tau \triangleq z_\tau - \mu_\tau$ and
$\zeta_\tau \triangleq \langle \theta, \mu_\tau\rangle + f_\tau(x_\tau) +
\xi_\tau$, just as before. Now we must control this error term, for
which we need both~\pref{lem:matrix_bernstein}
and~\pref{lem:self_normalized}. Apply~\pref{lem:matrix_bernstein}
to the vectors $Z_\tau$, setting
$\hat{\Sigma}_t \triangleq \sum_{\tau=1}^{t-1}Z_\tau Z_\tau\tran $ and
$\Sigma_t \triangleq \sum_{\tau=1}^{t-1}\EE[Z_\tau Z_\tau \mid \Fcal_\tau]$. With
probability at least $1-\delta/(2T)$, we have that for all unit
vectors $v \in \RR^d$
\begin{align*}
v\tran \Sigma_t v \le 2v\tran \hat{\Sigma}_tv + 9d\log(9t) + 8\log(4T/\delta) \leq 2v\tran \hat{\Sigma}_tv + 9d\log(9T) + 8\log(4T/\delta).
\end{align*}
This implies a lower bound on all quadratic forms involving
$\hat{\Sigma}_t$, which leads to positive semidefinite inequality
\begin{align*}
\lambda I + \hat{\Sigma}_t\succeq (\lambda - 3d\log(9T)-8/3 \log(4T/\delta))I + (\hat{\Sigma}_t + \Sigma_t)/3.
\end{align*}
This means that for any vector $v$, we have
\begin{align*}
  \|v\|_{(\lambda I+\hat{\Sigma}_t)^{-1}}^2 & \leq \|v\|_{((\lambda - 3d\log(9T)-8/3 \log(4T/\delta))I + (\hat{\Sigma}_t+\Sigma_t)/3)^{-1}}^2\\
  & \leq 3\|v\|_{((3\lambda - 9d\log(9T) - 8\log(4T/\delta))I + \hat{\Sigma}_t+\Sigma_t)^{-1}}^2.
\end{align*}
Before we apply~\pref{lem:self_normalized}, we must introduce
the range parameter $M$. Fix a round $t$ and let $A \triangleq ((3\lambda -
9d\log(9T) - 8\log(4T/\delta))I + \hat{\Sigma}_t + \Sigma_t)$ denote
the matrix in the Mahalanobis norm. Then,
\begin{align*}
\left  \|\sum_{\tau=1}^{t-1} Z_\tau\zeta_\tau\right\|_{A^{-1}}^2 = M^2\left \|\sum_{\tau=1}^{t-1} Z_\tau\zeta_\tau\right\|_{(M^2A)^{-1}}^2.
\end{align*}
Now apply~\pref{lem:self_normalized} with $Q \triangleq M^2(3\lambda -
9d\log(9T) - 8\log(4T/\delta))I$. Since we require $Q \succ 0$, this
requires $\lambda > 3d\log(9T) - 8/3\log(4T/\delta)$, which is
satisfied under the preconditions for the lemma. Under this
assumption, we get
\begin{align*}
  \|\sum_{\tau=1}^{t-1} Z_\tau\zeta_\tau\|_{(\lambda I + \hat{\Sigma}_t)^{-1}}^2 & \leq 3M^2\|\sum_{\tau=1}^{t-1} Z_\tau\zeta_\tau\|_{(Q + M^2(\hat{\Sigma}_t + \Sigma_t))^{-1}}^2 \\\
  & \leq 6M^2\log\left(\frac{4T}{\delta}\sqrt{\frac{\det(Q+M^2(\hat{\Sigma}_t + \Sigma_t))}{\det(Q)}}\right),
\end{align*}
with probability at least $1-\delta/(2T)$. With a union bound, the
inequality holds simultaneously for all $T$, with probability at least
$1-\delta$.

The last step is to analyze the determinant. Using the same argument
as in the proof of~\pref{lem:det_to_trace}, it is not hard to
show that
\begin{align*}
\left(\frac{\det(Q + M^2(\hat{\Sigma}_t + \Sigma_t))}{\det(Q)}\right)^{1/d} \leq 1 + \frac{2(t-1)}{d(3\lambda - 9d\log(9T) - 8\log(4T/\delta))}.
\end{align*}
If we impose the slightly stronger condition that $\lambda \ge
4d\log(9T) + 8\log(4T/\delta)$, then the term in the denominator is at
least $1$, and then we have that
\begin{align*}
\|\hat{\theta}_t - \theta\|_{\Gamma_t} \leq \sqrt{\lambda} + \sqrt{6M^2\log(4T/\delta) + 3dM^2\log(1 + 2T/d)}.
\end{align*}
Finally, as in the two-action case, we use the fact that $|\zeta_t| \le
3 \triangleq M$.
\end{proof}

Recall the setting of
$\ellipsoid \triangleq \sqrt{\lambda} + \sqrt{27d\log(1 + 2T/d) +
  54\log(4T/\delta)}$ and the definition of
$\lambda \triangleq 4d\log(9T) + 8\log(4T/\delta)$. For the remainder of the
proof, condition on the probability $1-\delta$ event that~\pref{lem:estimator_general} holds. We now turn to analyzing the
regret.
\begin{lemma}
\label{lem:regret_transform_general}
Let $\mu_t \triangleq \EE_{a \sim \pi_t}z_{t,a}$ where $\pi_t$ is the solution
to~\pref{eq:optimize} and assume the conclusion of~\pref{lem:estimator_general} holds. Then with probability at
least $1-\delta$
\begin{align*}
\reg(T)\leq (1+6\ellipsoid)\sqrt{2T\log(2/\delta)} + 3\ellipsoid\sqrt{T\sum_{t=1}^T\tr(\Gamma_t^{-1}(z_{t,a_t} - \mu_t)(z_{t,a_t}-\mu_t)\tran )}.
\end{align*}
\end{lemma}
This lemma is slightly more complicated than~\pref{lem:selection_two_arm}.
\begin{proof}
First, using the same application of Azuma's inequality as in the
proof of~\pref{lem:selection_two_arm}, with probability
$1-\delta/2$, we have
\begin{align*}
\reg(T)\leq \sqrt{2T\log(2/\delta)} + \sum_{t=1}^T \EE_{a \sim \pi_t} [\langle \theta, z_{t,a_t^\star} - z_{t,a} \rangle \mid \Fcal_{t}].
\end{align*}
Now we work with this latter expected regret
\begin{align*}
\sum_{t=1}^T \EE_{a \sim \pi_t} [\langle \theta, z_{t,a_t^\star} - z_{t,a}\rangle \mid \Fcal_{t}] = \sum_{t=1}^T \langle \theta, z_{t,a_t^\star} - \mu_t \rangle \leq \sum_{t=1}^T \langle \hat{\theta}, z_{t,a_t^\star} - \mu_t\rangle + \ellipsoid\| z_{t,a_t^\star} - \mu_t\|_{\Gamma_t^{-1}}.
\end{align*}
For the first term, we use the filtration condition~\pref{eq:filter}
\begin{align*}
\langle \hat{\theta}, z_{t,a_t^\star} - \mu_t\rangle &= \sum_{a \in \Acal_t} \pi_t(a) \langle \hat{\theta},z_{t,a_t^\star} - z_{t,a}\rangle \leq \ellipsoid \sum_{a \in \Acal_t} \pi_t(a) \|z_{t,a_t^\star} - z_{t,a}\|_{\Gamma_t^{-1}}\\
& \leq \ellipsoid\|z_{t,a_t^\star} - \mu_t\|_{\Gamma_t^{-1}} + \ellipsoid \sum_{a \in \Acal_t}  \pi_t(a) \|z_{t,a} - \mu_t\|_{\Gamma_t^{-1}}.
\end{align*}
Applying the feasibility condition in~\pref{eq:optimize}, we
can bound the expected regret by
\begin{align*}
\sum_{t=1}^T \EE_{a \sim \pi_t} [\langle \theta, z_{t,a_t^\star} - z_{t,a}\rangle \mid \Fcal_{t}] \leq 3\ellipsoid \sum_{t=1}^T \sqrt{\tr(\Gamma_t^{-1} \Cov_{a \sim \pi_t}(z_{t,a}))} \leq 3\ellipsoid\sqrt{T \sum_{t=1}^T \tr(\Gamma_t^{-1} \Cov_{a \sim \pi_t}(z_{t,a}))}.
\end{align*}
To complete the proof, we need to relate the covariance, which takes
expectation over the random action, with the particular realization in
the algorithm, since this realization affects the term
$\Gamma_{t+1}$. Let $Z_t \triangleq z_{t,a_t} - \mu_t$ denote the
centered realization, then the covariance term is
\[
  \Cov_{a\sim \pi_t}(z_{t, a}) = \EE[ Z_t Z_t\tran\mid \Fcal_{t} ]
\]
In order to derive a bound on
$\sum_{t=1}^T \tr(\Gamma_t^{-1} \Cov_{a \sim \pi_t}(z_{t,a}))$, we
first consider the following
\begin{align*}
\sum_{t=1}^T\tr(\Gamma_{t}^{-1}\EE[Z_tZ_t\tran  \mid \Fcal_t]) - \tr(\Gamma_t^{-1} Z_tZ_t\tran ).
\end{align*}
Observe that sequence of sums
  $\{\sum_{t=1}^\tau \tr(\Gamma_{t}^{-1}\EE[Z_tZ_t\tran \mid \Fcal_t]) -
  \tr(\Gamma_t^{-1} Z_tZ_t\tran )\}_{\tau=1}^T$ is a martingale. Also, each term
$\tr(\Gamma_{t}^{-1}\EE[Z_tZ_t\tran \mid \Fcal_t]) - \tr(\Gamma_t^{-1}
Z_tZ_t\tran )$ is bounded by $1$ because $\Gamma_1 = \lambda I$ and
$\lambda > 1$. Applying the Freedman's inequality reveals that with
probability at least $1-\delta/2$
\begin{align*}
\sum_{t=1}^T\tr(\Gamma_{t}^{-1}\EE[Z_tZ_t\tran  \mid \Fcal_t]) - \tr(\Gamma_t^{-1} Z_tZ_t\tran ) & \leq 2\sqrt{\sum_{t=1}^T\EE[(Z_t\tran \Gamma_t^{-1}Z_t)^2 \mid \Fcal_t]\log(2/\delta)} + 2\log(2/\delta)\\
& \leq 2\sqrt{\sum_{t=1}^T\tr(\Gamma_t^{-1}\EE[Z_tZ_t\tran \mid\Fcal_t]) \log(2/\delta)} + 2\log(2/\delta)\\
& \leq \frac{1}{2}\sum_{t=1}^T\tr(\Gamma_t^{-1}\EE[Z_tZ_t\tran \mid\Fcal_t]) + 4\log(2/\delta).
\end{align*}
Then rearranging and plugging back into our regret bound, we have
\begin{align*}
\reg(T)&\leq \sqrt{2T\log(2/\delta)} + 3\ellipsoid\sqrt{2T \left(\sum_{t=1}^T\tr(\Gamma_t^{-1}Z_tZ_t\tran ) + 4\log(2/\delta)\right)}\\
& \leq (1+6\ellipsoid)\sqrt{2T\log(2/\delta)} + 3\ellipsoid\sqrt{2T\sum_{t=1}^T\tr(\Gamma_t^{-1}Z_tZ_t\tran )}. \tag*\qedhere
\end{align*}
\end{proof}

To conclude the proof of the theorem, apply~\pref{lem:regret},
which applies on the last term on the RHS of~\pref{lem:regret_transform_general}. Overall, with probability at
least $1-2\delta$, we get
\begin{align*}
\reg(T)\leq (1+6\ellipsoid)\sqrt{2T\log(2/\delta)} + 3\ellipsoid\sqrt{2 T d(1+1/\lambda) \log(1+T/(d\lambda))}.
\end{align*}
Since $\lambda = \Theta(d\log(T/\delta))$ and
$\ellipsoid = \order(\sqrt{d \log(T)} + \sqrt{\log(T/\delta)})$, we
get with probability $1- \delta$,
\begin{align*}
\reg(T)\leq \order\left(d\sqrt{T}\log(T) + \sqrt{dT\log(T)\log(T/\delta)} + \sqrt{T\log(T/\delta)\log(1/\delta)}\right).
\end{align*}


\bibliographystyle{plainnat}
\bibliography{double_ucb}

\end{document}